\documentclass{article}

\PassOptionsToPackage{numbers, compress}{natbib}

\usepackage[preprint]{neurips_2026}

\usepackage[utf8]{inputenc} 
\usepackage[T1]{fontenc}    
\usepackage{hyperref}       
\usepackage{url}            
\usepackage{booktabs}       
\usepackage{amsfonts}       
\usepackage{nicefrac}       
\usepackage{microtype}      
\usepackage{xcolor}         

\usepackage{amsmath}
\usepackage{amssymb}
\usepackage{mathtools}
\usepackage{amsthm}
\usepackage{acronym}
\usepackage{enumitem}
\usepackage{algorithm}
\usepackage{algpseudocode}
\usepackage{subcaption}
\usepackage{multirow}

\theoremstyle{plain}
\newtheorem{theorem}{Theorem}[section]
\newtheorem{proposition}[theorem]{Proposition}
\newtheorem{lemma}[theorem]{Lemma}

\theoremstyle{definition}

\theoremstyle{remark}

\usepackage[all=normal,paragraphs=tight,floats=normal,mathspacing=normal,wordspacing=tight,charwidths=tight,mathdisplays=normal,leading=normal]{savetrees}

\newcommand{\myVec}[1]{{\boldsymbol{#1}}}
\newcommand{\myMat}[1]{{\boldsymbol{#1}}}
\newcommand{\mySet}[1]{\mathcal{#1}}

\let\oldnl\nl
\newcommand{\nonl}{\renewcommand{\nl}{\let\nl\oldnl}}%

\acrodef{ai}[AI]{artificial intelligence}
\acrodef{tta}[TTA]{test-time adaptation}
\acrodef{dl}[DL]{deep learning}
\acrodef{psd}[PSD]{positive semi-definite}
\acrodef{dnn}[DNN]{deep neural network}
\acrodef{cnn}[CNN]{convolutional neural network}
\acrodef{dcnn}[DCNN]{deconvolutional neural network}
\acrodef{mlp}[MLP]{multi-layer perceptron}
\acrodef{snr}[SNR]{signal-to-noise ratio}
\acrodef{awgn}[AWGN]{additive white Gaussian noise} 
\acrodef{ml}[ML]{machine learning} 
\acrodef{sgd}[SGD]{stochastic gradient descent} 
\acrodef{mse}[MSE]{mean-squared error}
\acrodef{rmse}[RMSE]{root mean squared error}
\acrodef{rmspe}[RMSPE]{root mean squared periodic error}
\acrodef{mle}[MLE]{maximum likelihood estimation}
\acrodef{snr}[SNR]{signal-to-noise ratio}
\acrodef{aoa}[AoA]{Angle of Arrival}
\acrodefplural{aoa}[AoAs]{Angles of Arrival}
\acrodef{em}[EM]{electromagnetic}
\acrodef{ula}[ULA]{uniform linear array}
\acrodef{em}[EM]{Electromagnetic}
\acrodef{mimo}[MIMO]{multiple-input multiple-output}
\acrodef{evd}[EVD]{eigenvalues decomposition}
\acrodef{bong}[BONG]{Bayesian online natural gradient}
\acrodef{ssm}[SSM]{state-space model}
\acrodef{ekf}[EKF]{extended Kalman filter}
\acrodef{ber}[BER]{bit error rate}
\acrodef{bbb}[BBB]{Bayes-by-backprop}
\acrodef{blr}[BLR]{Bayesian learning rule}
\acrodef{bog}[BOG]{Bayesian online gradient}
\acrodef{bce}[BCE]{binary cross-entropy}
\acrodef{elbo}[ELBO]{evidence lower bound}
\acrodef{ou}[OU]{Ornstein--Uhlenbeck}
\acrodef{dlr}[DLR]{diagonal plus low-rank}
\acrodef{cfo}[CFO]{carrier frequency offset}

\acrodef{name}[AURA]{Adaptive Update through Representation Adaptation}

\title{Online Learning via Learned Latent Bayesian Tracking}

\author{%
  Guy Gerson\\
  School of ECE\\
  Ben-Gurion University\\
  Be`er-Sheva, Israel \\
  \texttt{guygers@post.bgu.ac.il} \\
  \And  
  Tomer Raviv\\
  School of ECE\\
  Ben-Gurion University\\
  Be`er-Sheva, Israel \\
  \texttt{tomerraviv95@gmail.com} \\
  \And
  Nir Shlezinger \\
  School of ECE\\
  Ben-Gurion University\\
  Be`er-Sheva, Israel \\
  \texttt{nirshl@bgu.ac.il}  \\ 
  \And
    Tirza Routtenberg \\
  School of ECE\\
  Ben-Gurion University\\
  Be`er-Sheva, Israel \\
  \texttt{tirzar@bgu.ac.il} \\
  \And
  Osvaldo Simeone \\
 Institute for Intelligent Networked Systems (INSI)\\ 
 Northeastern University London \\
  London,  U.K \\ 
  \texttt{o.simeone@northeastern.edu}\\
}

\begin{document}

\maketitle

\begin{abstract}
Online learning in non-stationary environments requires models to adapt rapidly from streaming data under strict computational constraints. A principled approach casts online learning as Bayesian state tracking, where model parameters are updated sequentially via Bayesian filtering. However, applying Bayesian filters directly to modern deep models is computationally prohibitive due to the high dimensionality of parameter space, forcing existing methods to rely on restrictive approximations or manually designed low-dimensional subspaces. In this work, we identify the absence of a suitable low-dimensional dynamical representation as the core bottleneck in Bayesian filtering-based online learning. 
Accordingly, we propose {\em \ac{name}}, a meta-learning framework that learns offline a low-dimensional latent state-space model governing the evolution of optimal model parameters under distribution shift.
Online adaptation is then performed via extended Kalman filtering in this learned latent space followed by the reconstruction of the full model parameters through a learned lifting map, enabling efficient single-step online adaptation while preserving model expressiveness.
Evaluated on online adaptation of neural wireless receivers under time-varying channels and on non-stationary image classification, \ac{name} shows substantial improvements in adaptation speed, accuracy, and computational efficiency over existing online learning and Bayesian filtering baselines, demonstrating that an adaptation-aware latent geometry is beneficial for effective Bayesian online learning in high-dimensional models.
\end{abstract}
\acresetall

\section{Introduction}
Modern machine learning systems are increasingly deployed in environments where the data-generating distribution evolves over time~\cite{cossu2026practical}. Such non-stationarity arises in a wide range of settings, including adaptive signal processing, wireless communications, and vision systems under distribution shift~\cite{galashov2024non}. In these scenarios, a learning model must continuously adapt to remain aligned with the current distribution, often under strict computational and latency constraints. These considerations pose the challenge of how to enable {\em rapid and reliable adaptation} from limited, streaming data~\cite{giannini2024streaming}.

A common approach to adaptation in non-stationary environments relies on online optimization, where model parameters are updated sequentially using stochastic gradient-based methods as new data arrives~\cite{duchi2011adaptive,hoi2021online}. While effective in many settings, such approaches lack a principled mechanism for uncertainty-aware updates and typically require multiple gradient steps and careful tuning to ensure stability.
This renders them poorly suited to regimes where adaptation must occur within a small number of samples under stringent computational and latency constraints.
An alternative paradigm, dating back over two decades~\cite{wan2000unscented}, casts online learning as a form of Bayesian state estimation (tracking), where the model parameters are treated as a state vector evolving over time and updated recursively using Bayesian filters. This perspective offers an appealing route toward rapid online adaptation, as Bayesian filters, and particularly Kalman-type filters, naturally yield principled single-step prediction--correction updates that propagate uncertainty~\cite{sarkka2023bayesian}. However, applying such methods directly to modern deep models is computationally prohibitive, since the complexity of Bayesian filtering scales at least quadratically (and often cubically) with the dimension of the state, i.e., the number of model parameters in our context. Consequently, existing Kalman-based online learning methods have resorted to two classes of approximations: $1)$ structural constraints on the covariance, such as diagonal~\cite{Chang2022on} or low-rank parameterizations~\cite{Chang2023low,jones2024bayesian,gusakov2025rapid}; and $2)$ representational restrictions, such as confining adaptation to predefined parameter subsets or relying on manually specified dynamical models~\cite{cartea2026detecting,titsias2024kalman,duran-martin2025BONE}. While these approximations reduce complexity, they substantially limit adaptation performance and often remain too costly for large-scale real-time learning.

\paragraph{Contributions.}
This work identifies the absence of a suitable \emph{latent dynamical state-space representation} as the core bottleneck in Kalman-based online learning. We argue that the main challenge is not the Bayesian filter itself, but rather the lack of a low-dimensional dynamical representation in which filtering can be performed efficiently without sacrificing model expressiveness.
Motivated by this observation, our contributions are as follows:
\begin{itemize}
\item \textbf{Latent Bayesian online adaptation.} 
We propose {\em \ac{name}}, a framework for rapid online learning,
which is performed in a learned low-dimensional latent space via single-step \ac{ekf}~\cite{schmidt1981kalman}, with the full model parameters recovered through a lifting map. 
 This decouples the computational cost of filtering, which scales with the latent dimension, from the expressiveness of the underlying model. The \ac{ekf} is selected for its simplicity and end-to-end differentiability~\cite{greenberg2023optimization,xu2024ekfnet}.


\begin{figure}
    \centering
    \includegraphics[width=\linewidth]{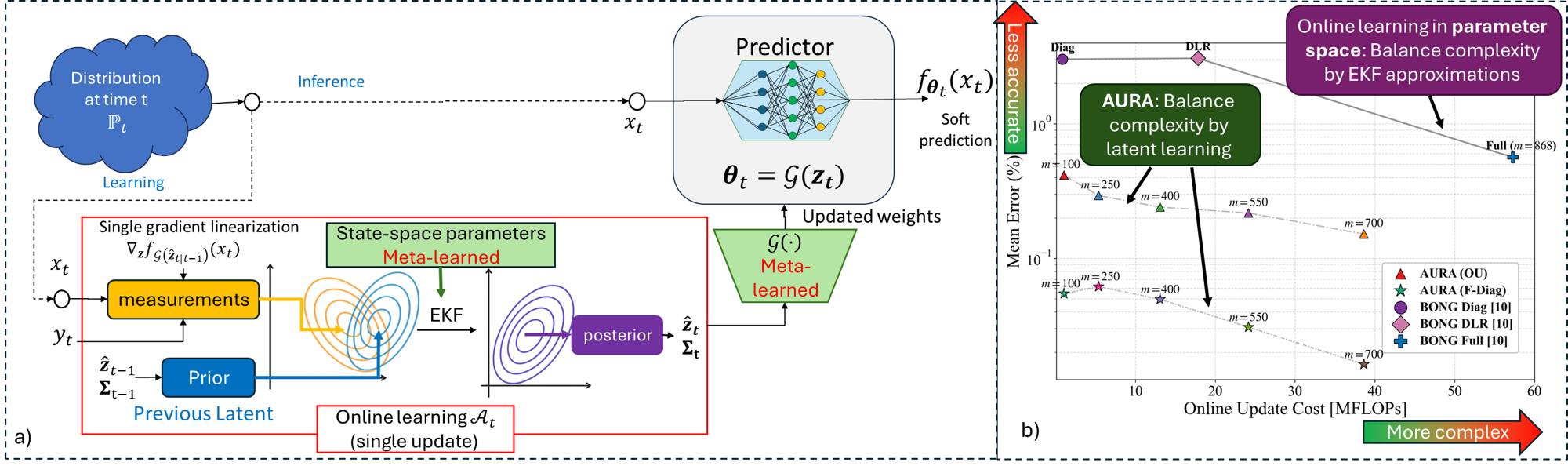}
    \caption{Proposed methodology: $a)$ \ac{name} pipeline via \ac{ekf} in meta-learned latent state-space representation (marked in \textcolor{red}{red} fonts); $b)$ Error--complexity trade-off for deep receiver adaptation, showing \ac{name}'s enhanced tradeoff compared to altering the learning algorithm in parameter space.
    }
    \label{fig:OverallView1}
\end{figure}


\item \textbf{End-to-end meta-learning of the latent adaptation space.} Rather than manually setting the state-space representation governing the \ac{ekf} operation~\cite{durbin2012time}, \ac{name} learns the latent dynamics, and the lifting map from offline non-stationary trajectories, optimizing them end-to-end for online adaptation performance. Accordingly, \ac{name} effectively meta-learns a geometry of parameter space tailored for rapid online adaptation. This yields the online learning paradigm illustrated in Fig.~\ref{fig:OverallView1}, where fast and reliable adaptation emerges from the interplay between structured Bayesian updates and a learned low-dimensional representation. 

\item \textbf{Empirical validation across two domains.} 
We demonstrate the effectiveness of \ac{name} in settings involving non-stationary data. In particular, we consider online adaptation of neural wireless receivers under time-varying channels,
where rapid adaptation is essential~\cite{raviv2023adaptive}, as well as image classification with non-stationary distortion. Our results show that \ac{name} consistently 
enables accurate adaptation using a single low-complexity update step per sample.
\end{itemize}

\section{Related Work}
\label{sec:related}

\paragraph{Streaming and Online Learning.}
Learning from non-stationary data streams is an active area of research~\cite{cossu2026practical}, where the primary objective is to rapidly adapt to evolving data distributions~\cite{hoi2021online}. In such settings, data arrives sequentially and may exhibit concept drift~\cite{lu2018learningFull}, rendering previously learned models obsolete. A wide range of approaches have been proposed, including online gradient-based methods~\cite{biehl1995learning}, adaptive optimizers~\cite{bartlett2007adaptive}, and proximal methods~\cite{dixit2019online}.  
These approaches focus on asymptotic performance, aiming to track the evolving optimum up to bounded error rather than ensuring reliable performance at each time instance, which can lead to suboptimal behavior under rapidly changing distributions~\cite{simonetto2020time}.
A related body of work considers adaptation settings in which the learner is provided with a small batch or task-specific dataset at each adaptation stage, rather than individual streaming samples. This includes incremental and continual learning methods, which update models sequentially across tasks or domains while mitigating catastrophic forgetting~\cite{parisi2019continual,de2021continual}, as well as \ac{tta} and continual \ac{tta} approaches that adapt deployed models using unlabeled test batches under distribution shift~\cite{wang2021tent,wang2022continual,niu2023towards,schirmer2025temporal}. Related paradigms further include gradual and temporal domain adaptation methods, which exploit temporally ordered shifts across adaptation domains~\cite{kumar2020understanding,bai2023temporal}. 
These approaches typically rely on batch-level statistics or iterative optimization over multiple samples, making them less suitable for streaming online with strict latency constraints.

\paragraph{Bayesian Filtering for Online Learning.}
A complementary line of work treats model parameters as a time-evolving state vector and performs adaptation via Bayesian filtering. Classical approaches based on Kalman filtering for training neural networks date back to early works on nonlinear Kalman extensions~\cite{wan2000unscented}. More recent works have extended this perspective to modern deep models, for instance, by reinterpreting stochastic gradient updates through Kalman-style uncertainty propagation~\cite{xia2025koala}, or by combining deep representations with state-space models for online classification~\cite{titsias2024kalman}. 
In parallel, several works develop filtering-based approaches for \emph{Bayesian neural networks}, including Kalman-based posterior tracking~\cite{gusakov2025rapid}, subspace and last-layer Bayesian updates~\cite{cartea2026detecting}, and natural-gradient-based variants~\cite{jones2024bayesian}. While these methods provide principled uncertainty-aware updates, their applicability to deep models is fundamentally limited by the computational complexity of Bayesian filtering, which scales at least quadratically with the number of parameters. As a result, proposed implementations rely on restrictive structural covariance approximations, such as diagonal~\cite{Chang2022on} or \ac{dlr} covariances~\cite{Chang2023low}, or confine adaptation to manually predefined subspaces~\cite{duran-martin2025BONE}.
In contrast, our approach explicitly learns a low-dimensional latent state-space and its transformation to parameter space tailored for online adaptation, enabling efficient Kalman-based updates while preserving the expressiveness of the underlying model.

\section{System Model and Problem Formulation}
\label{sec:System Model and Preliminaries}

\paragraph{Time-Varying Data Distribution.}
We consider a multi-class classification problem with an input space $\mathcal{X} \subseteq \mathbb{R}^D$ and a finite label space $\mathcal{Y} = \{1,\ldots,C\}$, in a sequential setting with temporal distribution shifts. Let $t$ denote a discrete time index, and assume that at each time step, an input-label pair $(x_t,y_t) \in \mathcal{X} \times \mathcal{Y}$ is drawn from a time-dependent distribution $\mathbb{P}_t(x,y)$. We consider a general form of non-stationarity in which the data-generating distribution may evolve continuously over time. This formulation captures a broad range of temporal variations, including gradual drifts as well as more abrupt changes in the statistics of the data-generating process. 

\paragraph{Predictive Model.}
Let $f_{\myVec{\theta}}:\mathcal{X}\to\Delta^{C-1}$ denote a predictive model parameterized by $\myVec{\theta}\in\mathbb{R}^d$, where $\Delta^{C-1}$ is the  probability simplex over the
label space $\mathcal{Y}$. Given an input $x\in\mathcal{X}$, the model outputs predicted class probabilities represented by the probability mass function $f_{\myVec{\theta}}(x) \in \Delta^{C-1}$ over $\mathcal{Y}$. 
As the underlying distribution evolves over time, a fixed parameterization $\myVec{\theta}$ may become suboptimal, necessitating adaptation to maintain accurate predictions.

\paragraph{Problem Formulation.} 
We consider a {\em streaming online learning} setting~\cite{cossu2026practical} where the predictive model must continuously adapt to a time-varying data-generating distribution. At each time step $t$, the learner observes a labeled sample $(x_t,y_t) \sim \mathbb{P}_t$ and updates the model parameters accordingly. Denoting by $\myVec{\theta}_t$ the model parameters after the update at time $t$, the adaptation is performed via an update rule 
\begin{equation}
\label{eqn:adaptation}
\myVec{\theta}_t = \mathcal{A}_t\big( x_t, y_t\big),
\end{equation}
where $\mathcal{A}_t$ is an online update operator which can have memory, e.g., depend on the previous parameterization $\myVec{\theta}_{t-1}$.  
The updated predictor $f_{\myVec{\theta}_t}$ is then used for inference under the current distribution $\mathbb{P}_t$. This formulation gives rise to two main challenges:
\begin{enumerate}[label={\em C\arabic*}]
    \item \label{itm:data} Adaptation must be carried out sequentially with limited data.
    \item \label{itm:rapid} The update rule must be computationally efficient and capable of rapid adaptation.
\end{enumerate}
To support the design of $\mathcal{A}_t$ under \ref{itm:data}-\ref{itm:rapid}, we assume access to an auxiliary labeled dataset comprising multiple non-stationary $T$-length trajectories.
Specifically, this dataset, which reflects the non-stationary nature of the environment, is given by
 \begin{equation} \label{eq:dataset} \mathcal{D} = \left\{ \mathcal{D}^{(i)} \right\}_{i=1}^{|\mathcal{D}|}, \qquad \mathcal{D}^{(i)} = \left\{ \left(x_{t,b}^{(i)},y_{t,b}^{(i)}\right) \right\}_{t=1,b=1}^{T,B}, \end{equation} 
 where each trajectory \(\mathcal{D}^{(i)}\) consists of \(T\) sequential time steps, and each time step \(t\) contains \(B\) labeled samples drawn from the corresponding time-varying distribution. These offline trajectories capture diverse distribution shifts and are used solely to design or learn the online adaptation mechanism prior to deployment. During deployment, the learner receives streaming samples sequentially and adapts online.





\section{Proposed Method: \ac{name}}
\label{sec:method} 
We now present \ac{name}, our proposed framework for rapid online learning in non-stationary environments. Our design is motivated by the observation that Bayesian filtering provides a principled mechanism for sequential uncertainty-aware adaptation from limited data (thus well
suited to \ref{itm:data}).  However, the direct application of Bayesian filtering to modern predictive models is computationally prohibitive due to the high dimensionality of parameter space (thus failing to address \ref{itm:rapid}). 
Accordingly, the key idea underlying \ac{name} is to perform online adaptation via Bayesian filtering in a \emph{learned low-dimensional latent space} rather than in the original parameter space. Specifically, we parameterize the evolving model parameters through a latent state-space representation, where online adaptation is conducted by tracking this latent state using a single-step \ac{ekf}-based update, and the adapted full model parameters are recovered through a learned lifting map (Subsection~\ref{ssec:Online}).
Since such a latent state-space representation is generally not available {\em{a priori}}, \ac{name} jointly learns the latent parameterization and its governing dynamics from offline non-stationary training trajectories (Subsection~\ref{ssec:Offline}). This results in a meta-learning procedure~\cite{hospedales2021meta} in which the representation is optimized specifically to facilitate rapid and reliable online adaptation under the update rule induced by Bayesian filtering.

\subsection{Latent Bayesian Online Adaptation}
\label{ssec:Online}
 At the core of \ac{name} lies the idea of performing online adaptation via Bayesian filtering in a learned low-dimensional latent space. Specifically, instead of trying to directly learn the desired high-dimensional model parameters for the current distribution $\mathbb{P}_t$, denoted by $\myVec{\theta}_t^\star \in \mathbb{R}^d$, we represent them through a latent variable (``state")  $\myVec{z}_t \in \mathbb{R}^m$, where $m \ll d$, together with a lifting map
$\mathcal{G}:\mathbb{R}^m \rightarrow \mathbb{R}^d$,
such that the desired parameters at time $t$ are given by
\[
\myVec{\theta}_t^\star = \mathcal{G}(\myVec{z}_t).
\]
The latent state $\myVec{z}_t$ is designed to capture the adaptation-relevant degrees of freedom of the predictor, enabling efficient online updates in a compact representation space.

\paragraph{Latent State-Space Representation.}
We model the evolution of the latent adaptation state via a stochastic linear dynamical system,
\begin{equation}
\myVec{z}_{t+1} = \myVec{F}\myVec{z}_t + \myVec{v}_t,
\label{eq:latent_dynamics}
\end{equation}
where $\myVec{F}\in\mathbb{R}^{m\times m}$ is a diagonal state transition matrix, and $\myVec{v}_t$ is 
process noise with covariance $\myMat{Q}$. This model provides a tractable surrogate for the temporal evolution of the adaptation state.
In particular, the setting $\myVec{F}=\gamma\myMat{I}$ with $\gamma <1$ is used in the common \ac{ou}-type dynamics~\cite{kurle2020}. 

The learner observes the latent state only indirectly through streaming labeled samples $(x_t,y_t)\sim\mathbb{P}_t$. We model the observation process as
\begin{equation}
{\text{OneHot}}(y_t)
=
f_{\mathcal{G}(\myVec{z}_t)}(x_t)
+
\myVec{w}_t,
\label{eq:observation_model}
\end{equation}
where ${\text{OneHot}}(\cdot)$ denotes one-hot encoding and $\myVec{w}_t$ captures the discrepancy between the predicted class probabilities of the desired model and the observed label. Motivated by the fact that, under cross-entropy training, the mismatch of the optimal probabilistic predictor is zero-mean and uncorrelated with the estimate, we model $\myVec{w}_t$ as zero-mean observation noise with covariance $\myMat{R}$.\footnote{Formal justification of this modeling assumption is provided in Appendix~\ref{app:proofs}.} 
is chosen as a regularized positive-definite covariance, which ensures numerical stability of the learning procedure despite the simplex structure of the observation representation.

\paragraph{EKF-Based Online Adaptation.}
Given the latent state-space representation \eqref{eq:latent_dynamics}--\eqref{eq:observation_model}, \ac{name} performs online adaptation via extended Kalman filtering.  The \ac{ekf} is especially  suited for the considered setting for several reasons: 
$(i)$ It performs adaptation using a {\em single} local linearization and update per time step, relying on one gradient (Jacobian) computation;
$(ii)$ While the complexity of the \ac{ekf} scales quadratically with the state dimension, here the state corresponds to the low-dimensional latent vector dimension ($m$) rather than the full parameter dimension ($d$), rendering the resulting updates feasible under \ref{itm:rapid}; 
$(iii)$ Compared with alternative Bayesian filters (e.g., unscented Kalman filters~\cite{wan2000unscented}, cubature Kalman filters~\cite{arasaratnam2009cubature}, or particle filters~\cite{djuric2003particle}), the \ac{ekf} admits a particularly simple and efficient implementation, while retaining strong performance in smooth nonlinear tracking.

The \ac{ekf} tracks the Gaussian approximation of the posterior distribution of the latent state via its mean $\hat{\myVec{z}}_t$ and covariance $\myMat{\Sigma}_t$. At each time step, it first predicts the latent state according to~\cite{durbin2012time}
\begin{equation}
\hat{\myVec{z}}_{t|t-1}
=
\myVec{F}\hat{\myVec{z}}_{t-1},
\qquad
\myMat{\Sigma}_{t|t-1}
=
\myVec{F}\myMat{\Sigma}_{t-1}\myVec{F}^\top+\myMat{Q}.
\label{eq:predict}
\end{equation}

The predicted states are then corrected using the incoming labeled observation via
\begin{equation}
\hat{\myVec{z}}_t
=
\hat{\myVec{z}}_{t|t-1}
+
\myMat{K}_t
\Big(
{\text{OneHot}}(y_t)
-
f_{\mathcal{G}(\hat{\myVec{z}}_{t|t-1})}(x_t)
\Big), \quad 
\myMat{\Sigma}_t
=
\myMat{\Sigma}_{t|t-1}
-
\myMat{K}_t\myMat{H}_t\myMat{\Sigma}_{t|t-1},
\label{eq:update}
\end{equation}
where
$
\myMat{H}_t
:=
\left.
{\partial f_{\mathcal{G}(\myVec{z})}(x_t)}/
{\partial \myVec{z}^\top}
\right|_{\myVec{z}=\hat{\myVec{z}}_{t|t-1}}
$
locally linearizes the observation model, and 
\begin{equation}
\myMat{K}_t
:=
\myMat{\Sigma}_{t|t-1}\myMat{H}_t^\top
\Big(
\myMat{H}_t\myMat{\Sigma}_{t|t-1}\myMat{H}_t^\top+\myMat{R}
\Big)^{-1}.
\label{eq:KG}
\end{equation}
The adapted predictor parameters are finally recovered through the lifting map as
$\myVec{\theta}_t
=
\mathcal{G}(\hat{\myVec{z}}_t)$.
We summarize the resulting online adaptation procedure in Algorithm~\ref{alg:latent_ekf}. At each incoming sample, \ac{name} performs a single  tracking step in latent space, followed by reconstruction of the full model parameters.

\begin{algorithm}
\caption{\ac{name} online learning at time $t$}
\label{alg:latent_ekf}
\begin{algorithmic}[1]
\Require Prior latent state $\{\hat{\myVec{z}}_{t-1}, \myMat{\Sigma}_{t-1}\}$, current data $(x_t,y_t)$,

\Require Latent state-space parameters $\myMat{F}, \myMat{Q}, \myMat{R}$; lifting map $\mathcal{G}(\cdot)$

\State Predict $\hat{\myVec{z}}_{t|t-1}$ and $\myMat{\Sigma}_{t|t-1}$ via \eqref{eq:predict}

\State Compute the Jacobian: $\myMat{H}_t = 
\left. {\partial f_{\mathcal{G}(\myVec{z})}(x_t)}/{\partial \myVec{z}^{\top}} \right|_{\myVec{z}=\hat{\myVec{z}}_{t|t-1}}$

\State Compute Kalman gain $\myMat{K}_t$ via \eqref{eq:KG}
\State Update $\hat{\myVec{z}}_t$ and $\myMat{\Sigma}_t$ via \eqref{eq:update}
\State \Return updated parameters $\myVec{\theta}_t = \mathcal{G}(\hat{\myVec{z}}_t)$
\end{algorithmic}
\end{algorithm}

\subsection{Offline Meta-Learning of the Latent Adaptation Space}
\label{ssec:Offline}
The effectiveness of Algorithm~\ref{alg:latent_ekf} critically depends on the choice of the latent state-space representation, namely, the lifting map $\mathcal{G}(\cdot)$ and the latent dynamics parameters governing \eqref{eq:latent_dynamics}. In particular, while the latent Bayesian formulation enables efficient online adaptation, its performance hinges on identifying a representation in which the desired model evolution is well captured by the assumed low-dimensional state-space representation.
To this end, \ac{name} learns the latent adaptation space offline from non-stationary training trajectories in \eqref{eq:dataset}. Specifically, we treat the components of the latent state-space model as \emph{meta-learnable hyperparameters} of the online adaptation process, and optimize them such that the \ac{ekf}-based online learner performs effectively under distribution shift.

\paragraph{Parameterized Latent State-Space Model.}
We parameterize the lifting map $\mathcal{G}:\mathbb{R}^m\rightarrow\mathbb{R}^d$ as a neural network with learnable parameters $\myVec{\psi}_G \in \mathbb{R}^n$. In addition, we treat the diagonal state transition matrix $\myVec{F}$ as a learnable parameter, and optionally also learn the noise covariance matrices $\myMat{Q},\myMat{R}$. Collectively, we denote the meta-learned latent state-space parameters by
$
\myVec{\phi}
:=
\{\myVec{\psi}_G,\myVec{F}, \myMat{Q},\myMat{R}\}$.
Under this parameterization, the online adaptation rule induced by Algorithm~\ref{alg:latent_ekf} is written as
$
\myVec{\theta}_t
=
\mathcal{A}_t\big(x_t,y_t;\myVec{\phi}\big)$,
highlighting its dependence on the learned latent state-space model.

\paragraph{Meta-Learning Objective.}
We employ  trajectory-based meta-learning  in which each non-stationary sequence in \eqref{eq:dataset}
is treated as an adaptation task. The goal is to meta-learn $\myVec{\phi}$ such that the induced Algorithm~\ref{alg:latent_ekf} achieves low cumulative loss along each trajectory, by treating the first sample at each time step as the support set and the remaining as the query set~\cite{hospedales2021meta}.
Accordingly, we optimize the meta-objective
\begin{equation}
\label{eq:meta_objective}
\mathcal{L}_{\mySet{D}}^{\rm meta}(\myVec{\phi}) = 
\frac{1}{|\mathcal{D}|}
\sum_{i=1}^{|\mathcal{D}|}
\frac{1}{T}
\sum_{t=1}^{T}
\frac{1}{B-1}
\sum_{b=2}^{B}
\ell\!\left(
f_{\mathcal{A}_t(x_{t,1}^{(i)},y_{t,1}^{(i)};\myVec{\phi})}
\big(x_{t,b}^{(i)}\big),
y_{t,b}^{(i)}
\right),
\end{equation}
which minimizes the empirical online prediction loss incurred by the latent Bayesian adaptation. 

\paragraph{End-to-End Optimization.}
Since both the EKF recursion and the lifting map $\mathcal{G}_{\myVec{\psi}_G}(\cdot)$ are differentiable~\cite{xu2024ekfnet}, the meta-objective in \eqref{eq:meta_objective} can be optimized end-to-end using stochastic gradient-based methods by differentiating through the online adaptation trajectory. This yields a training procedure in which the latent representation and its governing dynamics are learned jointly so as to maximize the effectiveness of the resulting online adaptation rule.
The meta-learning process, summarized as Algorithm~\ref{alg:meta_learning}, can be interpreted as learning a latent geometry of parameter space specifically tailored for rapid adaptation.

\begin{algorithm}
\caption{Offline meta-learning of \ac{name} 
}
\label{alg:meta_learning}
\begin{algorithmic}[1]
\Require Offline dataset $\mathcal{D}=\{\mathcal{D}^{(i)}\}$, learning rate $\eta$,  initial  state-space parameters $\myVec{\phi}, \hat{\myVec{z}}_0, \myMat{\Sigma}_0$
\Repeat
    \State Sample mini-batch of trajectories $\{\mathcal{D}^{(i)}\}$
    
    \For{each trajectory $\mathcal{D}^{(i)}=\{(x_{t,b}^{(i)},y_{t,b}^{(i)})\}_{t,b=1,1}^{T,B}$}
        
        \For{$t=1,\dots,T$}
            \State Run Algorithm~\ref{alg:latent_ekf} on $(x_{t,1}^{(i)},y_{t,1}^{(i)})$ to obtain $\myVec{\theta}_t^{(i)} = \mathcal{A}_t(x_{t,1}^{(i)},y_{t,1}^{(i)};\myVec{\phi})$ 
        \EndFor
    \EndFor
    

    \State Update latent state-space parameters:
    $    \myVec{\phi}
    \leftarrow
    \myVec{\phi}
    -
    \eta
    \nabla_{\myVec{\phi}}
    \mathcal{L}_{\{\mathcal{D}^{(i)}\}_{i=1}^{B}}^{\rm meta}(\myVec{\phi})$
    
\Until{convergence}

\State \Return Learned latent state-space parameters $\myVec{\phi}$
\end{algorithmic}
\end{algorithm}

\subsection{Discussion and Practical Considerations}
\label{ssec:Discussion}

\paragraph{Complexity Analysis.}
We next characterize the per-step computational complexity of Algorithm~\ref{alg:latent_ekf}. The dominant computational steps are as follows:
$(i)$ \emph{Jacobian computation:} evaluating the Jacobian $\myMat{H}_t$ requires differentiating through the composition $\myVec{z}\mapsto \mathcal{G}(\myVec{z}) \mapsto f_{\mathcal{G}(\myVec{z})}(x_t)$, incurring complexity $\mathcal{O}(n)$ (as $\mathcal{G}$ is implemented via an $n$-parameter neural network);
$(ii)$ \emph{Prediction step:} propagating the latent covariance in \eqref{eq:predict} requires 
$\mathcal{O}(m^2)$ operations (as  $\myVec{F}$ is constrained to be diagonal);
$(iii)$ \emph{Kalman update:} forming the Kalman gain and updating the posterior moments requires $\mathcal{O}(m^2)$ operations, since the output dimension $C$ is typically small and independent of the model size;
$(iv)$ \emph{Parameter reconstruction:} evaluating $\myVec{\theta}_t=\mathcal{G}(\hat{\myVec{z}}_t)$ requires a forward pass through the lifting map, with complexity $\mathcal{O}(n)$.
In general, the complexity per-step of \ac{name} is of the order of
$\mathcal{O}(n+m^2)$.
Since $m\ll d$ by design, this is substantially more efficient than applying Bayesian filtering directly in the ambient parameter space, whose complexity scales at least quadratically in the number of model parameters $d$.

\paragraph{Design Trade-Offs.}
\ac{name} exposes several design parameters that govern the trade-off between performance and complexity: 
$(i)$ The latent dimension $m$ controls the expressiveness of the adaptation space: larger latent spaces can capture richer parameter variations but incur higher  complexity; 
$(ii)$ Constraining the structure of the state transition  $\myVec{F}$ (e.g., to be a scaled identity matrix) provides a mechanism for reducing complexity; 
$(iii)$ The architecture of the lifting map $\mathcal{G}(\cdot)$ controls the expressive power of the parameter reconstruction, with larger  parameterizations increasing both representational flexibility and computational cost. These design choices allow \ac{name} to be tailored to a broad range of  constraints.

\paragraph{Interpretation and Theoretical Motivation.}
Although \ac{name} employs Bayesian filtering as its online adaptation mechanism, it is not intended as a Bayesian neural network method and does not seek to maintain a posterior distribution over the full predictive model, as is common in Kalman-based Bayesian learning~\cite{jones2024bayesian,duran-martin2025BONE,gusakov2025rapid}. Rather, the Bayesian posterior maintained by \ac{name} pertains solely to the low-dimensional latent state, which is treated as a stochastic tracking variable used to guide online updates. The predictive model parameters themselves are deterministically instantiated via the lifting map  $\myVec{\theta}_t=\mathcal{G}(\hat{\myVec{z}}_t)$.  \ac{name} thus leverages Bayesian filtering as a structured mechanism for efficient online optimization, rather than for uncertainty quantification over network weights~\cite{gawlikowski2023survey}.

The use of latent Bayesian tracking in \ac{name} is motivated by the hypothesis that the evolution of optimal model parameters 
often admits a structured low-dimensional representation. Such a hypothesis is consistent with empirical observations that overparameterized neural networks often admit substantial redundancy and can be effectively adapted within low-dimensional subspaces or low intrinsic-dimensional parameterizations~\cite{li2018measuring,aghajanyan2021intrinsic}. It is further aligned with classical dimensionality reduction results, such as the Johnson--Lindenstrauss lemma~\cite{johnson1984extensions}, which imply that structured high-dimensional point sets can often be embedded into substantially lower-dimensional spaces while approximately preserving their geometry. Under suitable regularity assumptions, one can further show that gradual changes in the data-generating distribution induce correspondingly smooth variations in compressed representations of the optimal model parameters, providing support for the latent dynamical systems viewpoint adopted by \ac{name}. Formal statements and proofs of these observations are provided in Appendix~\ref{app:proofs}.

\section{Experimental Study}
\label{sec:experiments}
We evaluate \ac{name} on two complementary settings, assessing its generality and practical relevance:
$(i)$  a non-stationary image classification task under gradually evolving corruptions~\cite{press2023rdumb,wang2022continual}; and $(ii)$
 online adaptation of neural wireless receivers, a motivating real-world application where rapid millisecond-level adaptation is essential for practical deployment under time-varying channels~\cite{lin2023overview,wiesmayr2025design,dai2020deep,qualcomm5gmmwaveMobility}.
Additional ablations, robustness studies, and proof-of-concept hardware experiments are provided in the appendix. The source code is available at \textcolor{red}{\url{https://github.com/aura-online-adaptation}}. 

\subsection{General-Purpose Online Adaptation for Non-stationary Image Classification}
\label{ssec:VisionExp}
We first evaluate \ac{name} on non-stationary image classification under evolving distribution shifts~\cite{press2023rdumb,wang2022continual}.
We compare \ac{name} not only with streaming online learning methods that update from labeled samples as they arrive~\cite{titsias2024kalman,bottou1998online}, but also with batch-wise adaptation methods commonly studied in the \ac{tta} literature~\cite{schneider2020corruptions,nado2020bn,mirza2022bn-dua,lim2023ttn,marsden2024roid}. Such methods accumulate multiple test samples for adaptation, and often rely on batch statistics or batch-level optimization objectives rather than single-sample streaming updates.
\begin{table}[t]
\centering
\caption{Mean classification error (\%, 3 seeds per cell) under gradual non-stationary corruptions.
Methods are grouped by adaptation regime. Best per dataset in \textbf{bold}; second best \underline{underlined}}
\label{tab:main2}
\small
\begin{tabular}{llcccc}
\toprule
Regime & Method & \# Train Samples & MNIST-C & CIFAR-10-C & CIFAR-100-C \\
\midrule
Pre-training
& Source
& 0
& 47.93 $\pm$ 19.83
& 41.81 $\pm$ 5.62
& 37.41 $\pm$ 4.82 \\
\midrule

\multirow{9}{*}{Streaming}
& \multirow{3}{*}{Online GD~\cite{bottou1998online}}
& 5  & 7.50 $\pm$ 0.50  & 29.46 $\pm$ 5.55 & 37.93 $\pm$ 1.83 \\
&  & 10 & 5.84 $\pm$ 0.24  & 25.89 $\pm$ 0.75 & 34.00 $\pm$ 2.37 \\
&  & 15 & 4.53 $\pm$ 0.39  & 22.37 $\pm$ 1.38 & 33.22 $\pm$ 3.35 \\
\cmidrule(lr){2-6}

& \multirow{3}{*}{EKF-FC~\cite{titsias2024kalman}}
& 5  & 19.74 $\pm$ 5.09 & 30.62 $\pm$ 1.34  & 31.27 $\pm$ 1.98 \\
&  & 10 & 17.00 $\pm$ 4.96 & 28.69 $\pm$ 4.61 & 33.38 $\pm$ 2.65 \\
&  & 15 &  12.16 $\pm$ 3.06 & 26.99 $\pm$ 5.26 & 32.43 $\pm$ 1.55 \\
\cmidrule(lr){2-6}

 & \multirow{3}{*}{\textbf{AURA}\textsuperscript{$\dagger$}}  
& 5  & 2.76 $\pm$ 0.63 & 16.13 $\pm$ 1.50 & 29.47 $\pm$ 0.62 \\
&  & 10 & 2.67 $\pm$ 0.93 & 16.84 $\pm$ 0.51 & 29.65 $\pm$ 1.97 \\
&  & 15 & \textbf{1.63 $\pm$ 0.09} & \textbf{14.28 $\pm$ 2.87} & \underline{26.66 $\pm$ 3.89} \\
\midrule

\multirow{4}{*}{Batch-wise}
& BN-Test~\cite{schneider2020corruptions} & 100 & 3.97 $\pm$ 2.58  & 19.46 $\pm$ 3.34 & 31.32 $\pm$ 2.08 \\
& BN-Alpha~\cite{mirza2022bn-dua} & 100 & 38.24 $\pm$ 14.23 & 28.05 $\pm$ 2.48 & 33.06 $\pm$ 1.15 \\
& BN-EMA~\cite{nado2020bn} & 100 & 3.17 $\pm$ 0.99  & 18.57 $\pm$ 2.64 & 33.19 $\pm$ 2.15 \\
& ROID~\cite{marsden2024roid} & 100 & \underline{2.64 $\pm$ 1.37} & \underline{16.00 $\pm$ 1.29} & \textbf{26.50 $\pm$ 1.37} \\
\bottomrule
\\ \multicolumn{6}{l}{\footnotesize $\dagger$ Latent dimension $m=256, 500, 500$ for MNIST-C, CIFAR-10-C,
  CIFAR-100-C.}
\end{tabular}
\end{table}

\paragraph{Experimental Setup.}
We consider pretrained image classifiers deployed under temporally evolving corruption streams. We generate non-stationary input streams by gradually varying the mixture of multiple corruption types over time~\cite{press2023rdumb,wang2022continual}, inducing smooth distribution drift across consecutive frames. We evaluate on MNIST-C~\cite{mu2019mnistc}, CIFAR-10-C, and CIFAR-100-C~\cite{hendrycks2019benchmarking}, which are corrupted variants of MNIST~\cite{lecun2010mnist} and CIFAR-10/100~\cite{krizhevsky2009cifar}, using standard pretrained backbone architectures for each benchmark based on ResNet-18~\cite{he2016resnet}, wide ResNet-28-10~\cite{zagoruyko2016wide}, and ResNet-29~\cite{xie2017aggregated}, respectively. 

We compare against three classes of baselines. First, we consider the original model with no adaptation. Second, among \emph{streaming supervised online learning} methods, we compare $(i)$ \ac{name};  $(ii)$ supervised online gradient descent (GD) 
\cite{bottou1998online};  and $(iii)$ parameter-space Bayesian tracking of the final classification layer  (EKF-FC)~\cite{titsias2024kalman}.  EKF-FC is applied with a full covariance  on MNIST-C and CIFAR-10-C, where the corresponding $d \times d$ matrix is small enough to enable parameter-space learning; on CIFAR-100-C the classifier layer has $d \approx 10^{5}$ parameters and the full covariance becomes infeasible, so we set the EKF-FC there with a diagonal covariance approximation. Applying Bayesian online learning to the \emph{entire} image-classifier parameter space, even under aggressive covariance approximations (e.g., diagonal~\cite{Chang2022on} or \ac{dlr}~\cite{Chang2023low}), is computationally prohibitive in this setting and is therefore omitted. Third, we compare with four representative \emph{batch-wise unsupervised adaptation} methods: three normalization-based \ac{tta} approaches---BN-Test, BN-EMA~\cite{schneider2020corruptions,nado2020bn}, and BN-Alpha~\cite{mirza2022bn-dua,lim2023ttn}---and ROID~\cite{marsden2024roid}, a self-supervised continual \ac{tta} method that adapts from unlabeled inputs via diversity-weighted soft-likelihood-ratio and consistency losses. Each test stream consists of $1000$ time steps, each delivering a batch of $100$ corrupted inputs. Within a step, up to $15$ inputs come with ground-truth labels and are used by the \emph{streaming supervised} methods (Online GD, EKF-FC, \ac{name}) for adaptation.

\paragraph{Results.}
Table~\ref{tab:main2} reports classification error and the per-method learning-sample budget. \ac{name} performs competitively across all benchmarks and achieves the best performance among streaming-feasible adaptation methods. In particular, \ac{name} attains the lowest overall error on both MNIST-C and CIFAR-10-C, outperforming every batch-wise baseline despite using as few as five adaptation samples per frame. On CIFAR-100-C it ranks second overall  and substantially improves over the streaming baselines, Online GD~~\cite{bottou1998online} and EKF-FC~~\cite{titsias2024kalman}.
ROID~\cite{marsden2024roid} attains the lowest error on CIFAR-100-C, but it is structurally incompatible with the streaming setting. Its loss combines a soft-likelihood-ratio term with diversity- and certainty-based weights and a class-prior EMA, all computed from the \emph{batch statistics} of the current frame and therefore requiring multiple unlabeled samples to be well-defined. Streaming ROID is consequently feasible only through external buffering, which induces a per-batch latency penalty proportional to the buffer size. In contrast, \ac{name} adapts through a single structured latent Bayesian update on the current sample and emits a prediction before the next arrival.

\begin{figure}
    \centering
    \includegraphics[width=\linewidth]{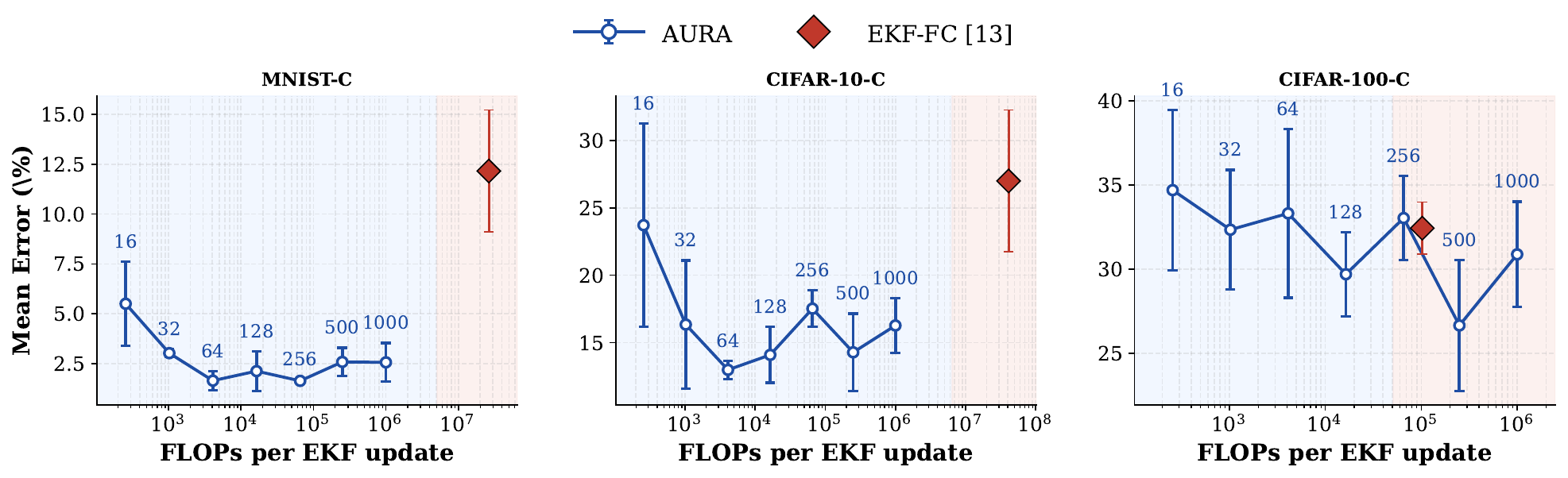}
    \caption{Complexity–performance trade-off in streaming online learning. 
    The blue curve sweeps AURA's latent dimension; the number above each marker is that $m$. 
    The red diamond is EKF-FC. Shaded regions separate the latent and parameter-space regimes. Error bars are $\pm 1$ standard deviation. 
    }
    \label{fig:vision_flops_tradeoff}
\end{figure}

\paragraph{Complexity--Performance Trade-off.}
To  assess the computational benefits of latent-space Bayesian adaptation, Fig.~\ref{fig:vision_flops_tradeoff} compares the streaming online learning complexity of \ac{name} and EKF-FC in floating-point operations (FLOPs) per update. This experiment serves both as a comparison against parameter-space adaptation and as an ablation of the complexity--performance trade-off controlled by $m$. The results show that \ac{name} consistently dominates parameter-space adaptation in the complexity--accuracy plane: for substantially lower computational cost, latent-space adaptation yields markedly better prediction accuracy than parameter-space filtering on the classifier layer. Moreover, varying  $m$ reveals the expected trade-off between expressiveness and efficiency: increasing the latent dimension initially improves performance by enriching the adaptation space, after which gains saturate and may slightly deteriorate due to over-parameterization or increased estimation noise. The selected operating points used in Table~\ref{tab:main2} are highlighted in Fig.~\ref{fig:vision_flops_tradeoff}, and are seen to lie near the empirical Pareto frontier for each dataset. These results indicate that learning a compact latent adaptation geometry makes Bayesian online learning  substantially more computationally efficient without sacrificing adaptation quality.

\subsection{Rapid Online Adaptation for Wireless Receivers}
\label{ssec:WirelessExp}
Next, we examine wireless communications, a practically important benchmark for evaluating rapid online adaptation 
\cite{raviv2023adaptive}. In particular, time-varying wireless channels constitute a prominent application where the inability to adapt neural receivers at millisecond timescales remains a major obstacle to practical deployment of neural physical layer communication systems~\cite{dai2020deep}.

\paragraph{Experimental Setup.}
We consider uplink multi-user detection with three single-antenna users and five receive antennas, using the DeepSIC neural receiver of~\cite{shlezinger2019deepsic} as the underlying predictive model. Channel realizations are obtained from QuaDRiGa~\cite{jaeckel2014quadriga} under the 3GPP Indoor Office scenario. We evaluate both linear and nonlinear channel models, where the latter incorporates hardware-induced distortions via a $\tanh$ nonlinearity (modeling nonlinearities  at the RF front-end~\cite{smaini2012rf}). 
We compare \ac{name} against representative streaming learning with online GD and \ac{ekf} variants in full parameter space which we term \ac{bong} following \cite{jones2024bayesian}, with full, diagonal, and \ac{dlr} covariance parameterizations~\cite{Chang2023low}. Offline meta-training and online evaluation are conducted on disjoint channel trajectories.
During online adaptation, the receiver is evaluated over temporally evolving channel sequences of $150$ frames. The first four frames serve for synchronization, during which 64 labeled pilot samples per frame are provided for initialization. Subsequently, the receiver enters a tracking regime in which only six labeled samples (pilots) are available per frame for adaptation, after which the adapted detector is used for inference on the remaining symbols, and evaluated in its \ac{ber}. Appendix~\ref{app:exp_details} contains full details of the training setup. 


\paragraph{Results.}
The results in Figs.~\ref{fig:ber_compare_linear_mixed}--\ref{fig:ber_compare_nonlinear_mixed} demonstrate that \ac{name} consistently outperforms all considered baselines across both linear and nonlinear QuaDRiGa scenarios. In particular, \ac{name} achieves the lowest \ac{ber} across the \ac{snr} sweep while maintaining substantially more stable temporal tracking throughout the channel evolution.
These gains are especially pronounced relative to full-parameter Bayesian filtering methods, highlighting that the benefit of \ac{name} stems not merely from employing Bayesian online updates, but from performing them in a learned low-dimensional latent adaptation space tailored for efficient tracking. This supports our central hypothesis that identifying an adaptation-aware latent geometry is critical for enabling effective Bayesian online learning in high-dimensional models.
Additional ablation studies analyzing the impact of latent dimension, latent dynamics parameterization, and robustness across channel models are provided in Appendix~\ref{app:additional}. Furthermore, Appendix~\ref{app:demo} presents a proof-of-concept hardware demonstration showing that \ac{name} enables millisecond-order self-adaptive neural WiFi receivers on software-defined radio platforms, validating the practical deployability of the proposed framework in real-time embedded systems.


\begin{figure}[t]
    \centering
    \begin{minipage}[t]{0.48\linewidth}
        \centering
        \includegraphics[
            width=\linewidth
        ]{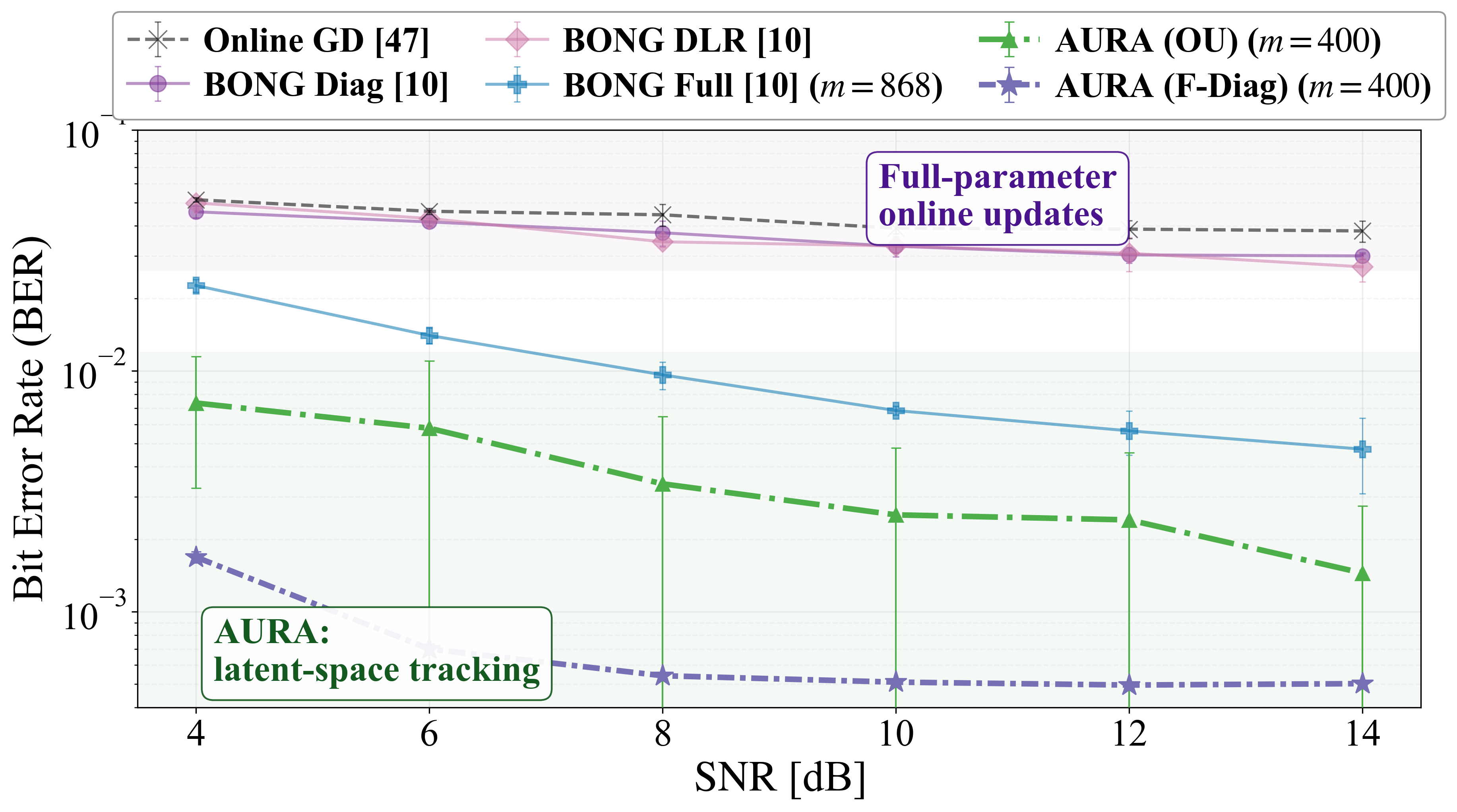}
        \caption{BER versus SNR, linear channel}
        \label{fig:ber_compare_linear_mixed}
    \end{minipage}
    \hfill
    \begin{minipage}[t]{0.48\linewidth}
        \centering
        \includegraphics[
            width=\linewidth
        ]{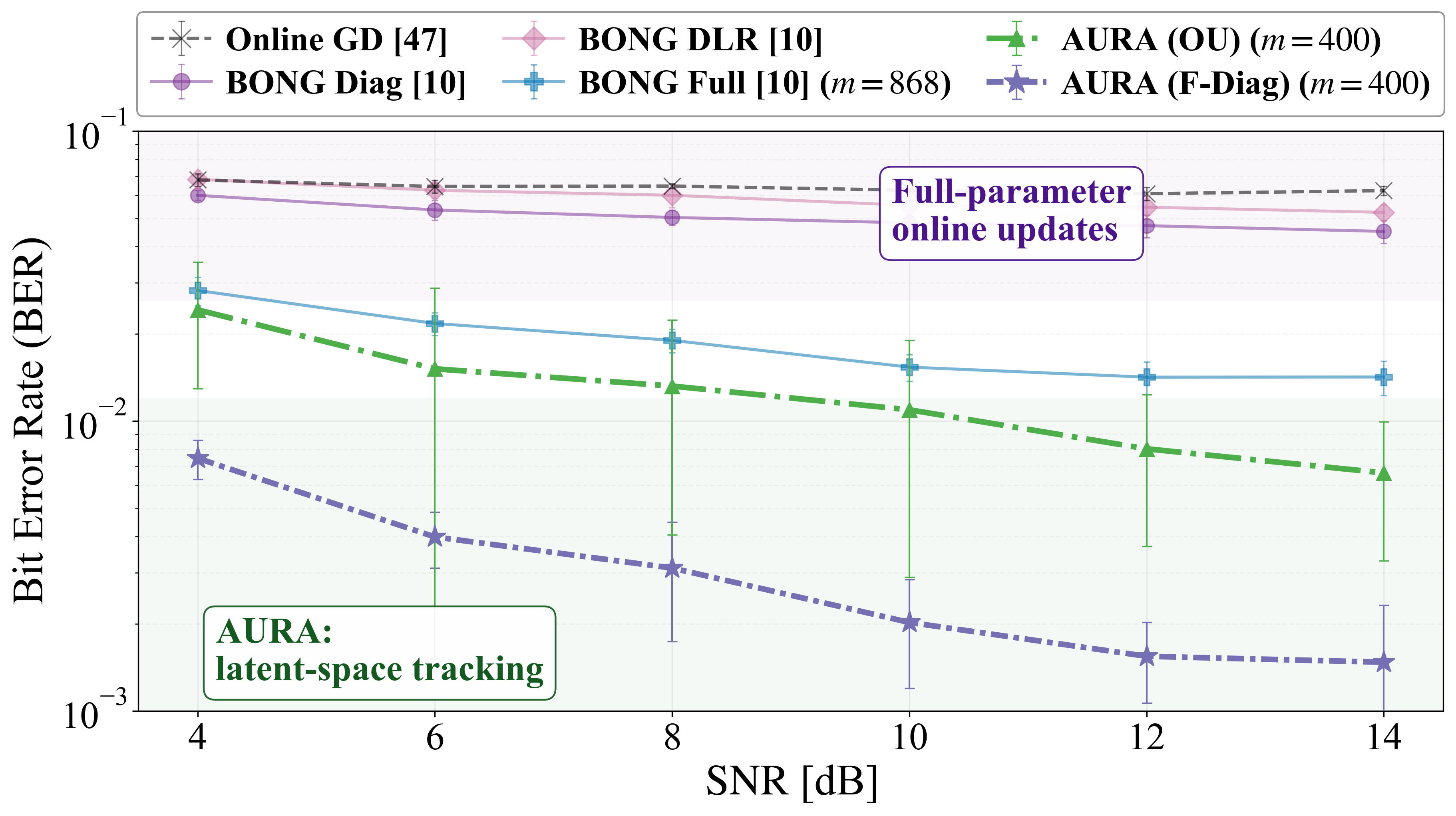}
        \caption{BER versus SNR, nonlinear channel}
        \label{fig:ber_compare_nonlinear_mixed}
    \end{minipage}
\end{figure}


\section{Conclusion}
\label{sec:conclusion}
We introduced \ac{name}, a framework for rapid online adaptation in non-stationary environments. By learning a low-dimensional state-space representation tailored for online adaptation, \ac{name} enables an efficient single-step Bayesian update. This yields a principled online learning mechanism that combines Kalman-based filtering with flexible learned representations. Experimental results demonstrate that \ac{name} achieves strong adaptation performance across diverse domains while operating under stringent computational constraints.  

\section*{Acknowledgments}

This work was supported by the Israel Science Foundation (ISF) under grant no.~3314/25, by the European Research Council (ERC) under the ERC Starting Grant No.~101163973 (FLAIR), and by the Israeli Ministry of Science and Technology. The work of O. Simeone was supported by the EPSRC (EP/W024101/1, EP/X011852/1) and by the ERC (No.~101198347).

\bibliographystyle{IEEEtran}
\bibliography{IEEEfull, mybib}

\newpage
\appendix
\section*{Appendix}

\newlength{\NumBoxWidth}
\setlength{\NumBoxWidth}{2em}

\newcommand{\outlineitem}[2]{%
  \noindent
  \hyperref[#1]{#2}%
  \nobreak\leaders\hbox to .6em{\hss.\hss}\hfill
  \nobreak\makebox[\NumBoxWidth][r]{\pageref*{#1}}%
  \par
}
\newcommand{\outlinesubitem}[2]{%
  \noindent\hspace*{1.5em}%
  \hyperref[#1]{#2}%
  \nobreak\leaders\hbox to .6em{\hss.\hss}\hfill
  \nobreak\makebox[\NumBoxWidth][r]{\pageref*{#1}}%
  \par
}

\subsubsection*{Appendix Overview}

{\setlength{\parskip}{0.2em}

\outlineitem{app:related_work}{Additional Related Work}

\outlineitem{app:proofs}{Theoretical Motivation for the Latent State-Space Formulation}
\outlinesubitem{app:latent_exist}{Existence of Low-Dimensional Model Representations}
\outlinesubitem{app:obs_noise}{Observation Model as Additive Noise Channel}
\outlinesubitem{app:bayesian_relation}{Relation to Bayesian Learning}
\outlinesubitem{app:latent_variation_proof}{Proof of Theorem~\ref{thm:LatentVariation}}
\outlinesubitem{app:LemmOrtProof}{Proof of Lemma~\ref{lem:Orthogonality}}
\outlinesubitem{app:bong_proof}{Proof of Lemma~\ref{lem:BONGtoKF}}

\outlineitem{app:exp_details}{Implementation Details of the Experiments}
\outlinesubitem{sec:settingWireless}{Non-stationary Wireless Receiver Adaptation}
\outlinesubitem{sec:settingVision}{Non-stationary Image Classification}

\outlineitem{app:additional}{Additional Experiments}
\outlinesubitem{app:wireless_ablation}{Additional Wireless Communication Experiments}
\outlinesubitem{app:vision_ablation}{Additional Non-stationary Image Classification Experiments}

\outlineitem{app:demo}{Proof-of-Concept: Adaptive Deep Software-Defined Radio}
\outlinesubitem{app:demo_setup}{Hardware Setup and Adaptive Receiver Architecture}
\outlinesubitem{app:demo_results}{Hardware Performance and Runtime Characterization}

\outlineitem{app:limitations}{Limitations}
}

\clearpage
\section{Additional Related Work}
\label{app:related_work}

\paragraph{Online Continual Learning.} 
Online continual learning (OCL) studies learning from sequential, potentially non-stationary data while retaining useful knowledge acquired from previous observations~\cite{de2021continual,cossu2026practical}. A central challenge is the stability-plasticity trade-off: adapting rapidly to new data without catastrophically overwriting previously acquired knowledge. Existing approaches address this challenge through mechanisms such as replay, regularization, and Bayesian posterior updates~\cite{nguyen2017variational,kurle2020}. \ac{name} shares with OCL the sequential, non-stationary setting and the need for rapid adaptation from limited observations, but differs in its primary objective. Rather than explicitly preserving performance on previously encountered distributions, \ac{name} aims to track the predictor that is best aligned with the current distribution. Accordingly, it does not rely on episodic replay or explicit consolidation. Instead, it meta-learns offline a latent parameterization and dynamical model that capture how effective predictors evolve, and uses them for uncertainty-aware Bayesian tracking at deployment. This distinction is particularly relevant in settings where past operating conditions become obsolete rather than tasks that must later be revisited.

\paragraph{Differentiable and Learned Kalman Filters.}
The offline meta-learning stage of \ac{name}, which learns latent adaptation dynamics by differentiating through Kalman-style recursions, is related to prior works that cast Kalman filtering as a differentiable machine learning model. Classical learning of state-space model parameters is typically studied separately from filtering within the framework of system identification, where the underlying dynamics are estimated from data~\cite{schon2011system,gedon2021deep}. More recent works instead optimize Kalman-type filters end-to-end by backpropagating through the filtering recursion to learn state-space parameters or noise statistics directly from task losses~\cite{greenberg2023optimization,xu2024ekfnet,haarnoja2016backprop, kloss2021train}. In parallel, a growing literature converts Kalman-type filtering itself into trainable neural architectures via learned gains~\cite{revach2022kalmannet}, learned corrections~\cite{satorras2019combining}, augmentations of prior~\cite{ghosh2023danse} and posterior updates~\cite{becker2019recurrent}, and broader AI-aided Kalman filtering frameworks~\cite{shlezinger2025artificial}. While these works primarily learn \emph{state estimators} for tracking latent physical processes, \ac{name} differs fundamentally in its objective: it does not learn a Kalman filter for state estimation, but rather leverages differentiable Bayesian filtering as a \emph{meta-learning mechanism} for constructing an online adaptation rule, where the tracked latent state represents a learned low-dimensional parameterization of the evolving optimal model parameters in a non-stationary learning problem.


\paragraph{Parameter-Efficient Fine-Tuning and Adaptation.}
The learned low-dimensional adaptation space of \ac{name} is also conceptually related to the literature on parameter-efficient fine-tuning (PEFT), which seeks to adapt large pretrained models (typically language models) by updating only a restricted subset or structured parameterization of the model weights rather than the full parameter vector. Representative PEFT approaches include adapter-based methods~\cite{houlsby2019parameter,karimi2021compacter}, prompt- and prefix-based tuning~\cite{li2021prefix,lester2021power}, selective parameter updates such as BitFit~\cite{zaken2022bitfit}, and low-rank or reparameterized adaptation mechanisms including LoRA and its variants~\cite{hu2022lora,dettmers2023qlora}; see also the recent comprehensive overview~\cite{wang2025parameter}. Despite the shared motivation of reducing adaptation dimensionality, the setting considered in this work is fundamentally different. PEFT methods typically address \emph{offline} adaptation to a fixed downstream task or domain, where substantial task-specific data and multiple optimization iterations are available. In contrast, \ac{name} addresses \emph{streaming online adaptation} in continuously evolving non-stationary environments, where only a handful of labeled samples are observed per update and adaptation must be performed under strict latency constraints. Consequently, while PEFT focuses on reducing the cost of conventional fine-tuning, \ac{name} instead learns a latent adaptation geometry specifically tailored for real-time Bayesian online tracking of drifting optima.




\section{Theoretical Motivation for the Latent State-Space Formulation}
\label{app:proofs}
\ac{name} is based on the principle that rapid adaptation can be achieved by \emph{tracking the evolution of optimal model parameters in a suitably learned latent space}. Concretely, we posit that, although  $\myVec{\theta}_t$ may be high-dimensional, it admits a low-dimensional representation whose temporal evolution is governed by smooth dynamics. When such a representation is available, online learning can be cast as a state estimation problem, where adaptation is carried out efficiently via Bayesian filtering in the latent space. In this appendix we establish the foundations of this approach. This includes:
$(i)$  providing a theoretical argument for the existence of low-dimensional representations in which the optimal parameters evolve smoothly over time; 
$(ii)$ identify conditions for which the representation of the observation corruption as additive uncorrelated noise in \eqref{eq:observation_model} holds; 
and $(iii)$ provide a rigorous relation between Kalman tracking in latent space and conventional Bayesian learning based on \ac{elbo} minimization. All proofs are delegated to the end of this appendix.

\subsection{Existence of Low-Dimensional Model Representations}
\label{app:latent_exist}
We begin by formalizing the learning objective of \ac{dnn} predictors and establishing the modeling viewpoint that motivates our formulation of a latent dynamical system.

\paragraph{Learning Framework.} 
Following standard statistical learning terminology~\cite{shalev2014understanding}, we consider learning guided by a loss function $\ell:\Delta^C \times \mathcal{Y} \mapsto \mathbb{R}$. The instantaneous risk at time $t$ is defined as
\begin{equation}
\mathcal{L}_t(\myVec{\theta})
:=
\mathbb{E}_{(x,y)\sim \mathbb{P}_t}
\left[\ell\left(f_{\myVec{\theta}}(x),y\right)\right].
\label{eq:instantaneous_risk}
\end{equation}
The learning objective at time $t$ is thus to minimize the risk induced by the current distribution.
Accordingly, we define the instantaneous optimal parameter vector as
\begin{equation}
\myVec{\theta}_t^\star \in \arg\min_{\myVec{\theta}} \mathcal{L}_t(\myVec{\theta}).
\label{eq:instantaneous_optimum}
\end{equation}

\paragraph{Motivating Learning Model. }
To ground the casting of online learning 
 as the tracking of a dynamic low-dimensional
system, we make use of the following assumption:
\begin{enumerate}[label={\em AS\arabic*},series=assumptions]
\item \label{itm:Compressibility}(Lipschitz Compressibility) There exists an  $L$-Lipschitz compression mapping $\mathcal{C}: \mathbb{R}^d \mapsto \mathbb{R}^{m}$ with $m\ll d$   such that $\myVec{\theta}_t^{\star}\in \mathbb{R}^d$ can be recovered from $\myVec{z}_t  = \mathcal{C} \big(\myVec{\theta}_t^{\star}\big) \in \mathbb{R}^m$ for each $t$. 
\end{enumerate}
Assumption \ref{itm:Compressibility} reflects the hypothesis that, although the ambient parameter space is high-dimensional, the set of optimal solutions induced by varying data distributions lies on a low-dimensional manifold.  In our setting, this suggests that, provided the sequence $\{\myVec{\theta}_t^\star\}$ remains sufficiently structured, there exists a low-dimensional representation $\myVec{z}_t$ that captures the relevant variations across time. Importantly, we do not require explicit knowledge of the mapping $\mathcal{C}$, but only posit its existence as a modeling assumption that will later be relaxed through learning.
Under \ref{itm:Compressibility}, we can identify regimes in which  the optimal parameters evolve along a low-dimensional manifold and vary similarly to the distribution variations. This is formulated in the following theorem:
\begin{proposition}
    \label{thm:LatentVariation}
    When the risk \eqref{eq:instantaneous_risk}  is differentiable and $\mu$-strongly quasi-convex, and the loss $\ell$ is uniformly bounded by some $B>0$, then, the latent representations $\myVec{z}_t$ and $\myVec{z}_{t+1}$  satisfy
    \begin{equation}
        \|\myVec{z}_t-\myVec{z}_{t+1}\|_2 \leq 2 L \sqrt{\frac{2}{\mu} B \cdot D_{\rm TV}(\mathbb{P}_t; \mathbb{P}_{t+1})}, \quad \forall t,
        \label{eq:LatentVariation}
    \end{equation}
    where $D_{\rm TV}(\cdot ; \cdot)$ is the total variation distance. 
\end{proposition}
While Proposition~\ref{thm:LatentVariation} is derived under restrictive conditions (e.g., strong quasi-convexity of the risk), its main implication is qualitative rather than prescriptive. Specifically, it suggests that the variation in the optimal parameters, when viewed through a suitable low-dimensional representation, is governed by the magnitude of the underlying distribution shift.  This observation provides support for modeling the latent sequence $\{\myVec{z}_t\}$ as a {\em smoothly evolving dynamical system}. Although the assumptions required for the theorem may not hold exactly in practical deep learning settings, they serve to motivate the broader hypothesis that non-stationary learning problems often admit low-dimensional structures with temporally coherent dynamics, which can be exploited for efficient online adaptation.

\subsection{Observations Model as an Additive Noise Channel}
\label{app:obs_noise}
The observation model \eqref{eq:observation_model} reflects the mismatch between the probabilistic output of the predictive model and the label, with $\myVec{w}_t$ being the estimation error. Specifically, when the loss $\ell$ guiding learning is the standard cross-entropy loss, then $\myVec{w}_t$ holds the following property:
\begin{lemma}
    \label{lem:Orthogonality}
      When $f_{\myVec{\theta}}$ is sufficiently expressive such that $\exists \myVec{\theta}_t \in \mathbb{R}^d$ for which $f_{\myVec{\theta}_t}(x)=[\mathbb{P}_t(y=1|x),\ldots,\mathbb{P}_t(y=C|x)]$ for all $x \in {\rm supp}(\mathbb{P}_t)$, then $\myVec{w}_t$ is zero-mean and uncorrelated with $f_{\mathcal{G}(\myVec{z}_t)}(x_t)$.
\end{lemma}
    Recall that 
   we design our online learning algorithm assuming that $\myVec{w}_t$ is zero-mean and uncorrelated with the model output. This modeling decision is thus further motivated by Lemma~\ref{lem:Orthogonality}.

\subsection{Relation to Bayesian Learning}
\label{app:bayesian_relation}
\ac{name} implements online learning by casting this task as the tracking of the latent representation of the optimal model parameters, rather than by using standard risk-minimization based learning formulation. Still, it admits a principled interpretation from a Bayesian learning perspective.  In particular, under a linearized Gaussian approximation, \ac{ekf} updates coincide with Bayesian online learning rules based on the \ac{elbo}. This is stated in the following lemma, adapted from~\cite{jones2024bayesian}:
\begin{lemma}
\label{lem:BONGtoKF}
    When $(i)$ the observation model induced by the predictor is replaced with its first-order Taylor approximation around $\hat{\myVec{z}}_{t|t-1}$, i.e., 
    $f_{\mathcal{G}(\myVec{z}_t)}(x_t) = 
    f_{\mathcal{G}(\hat{\myVec{z}}_{t|t-1})}(x_t)  + \myMat{H}_t(\myVec{z}_t - \hat{\myVec{z}}_{t|t-1})$; 
    and $(ii)$ the noise terms $\myVec{w}_t$ and $\myVec{v}_t$  are Gaussian, then \eqref{eq:predict}-\eqref{eq:update} are equivalent to a natural gradient minimization of the \ac{elbo} loss with Gaussian prior~\cite{khan2023bayesian}.
\end{lemma}
Lemma~\ref{lem:BONGtoKF} shows that while our learning methodology is based on Kalman-style adaptation, designed for tracking dynamic systems, it can viewed as specialization of existing Bayesian machine learning frameworks derived from conventional training objectives.

\subsection{Proofs}
\label{app:DetProofs}
\subsubsection{Proof of Proposition~\ref{thm:LatentVariation}}
\label{app:latent_variation_proof}
We first state the additional assumption introduced in the statement of the theorem:
\begin{enumerate}[label={\em AS\arabic*},series=assumptions, resume]
\item \label{itm:Convexity}(Strong quasi-convexity) The risk  \eqref{eq:instantaneous_risk}  is differentiable and $\mu$-strongly quasi-convex for some constant $\mu > 0$, i.e., for any $\myVec{\theta} \in \mathbb{R}^d$ and $t$ it holds that
\begin{equation*}
f_{\mathbb{P}_t}(\myVec{\theta}_t^{\star}) \geq f_{\mathbb{P}_t}(\myVec{\theta}) + \langle \nabla f_{\mathbb{P}_t}(\myVec{\theta}), \myVec{\theta}_t^{\star} -\myVec{\theta}\rangle + \frac{\mu}{2} \|\myVec{\theta}_t^{\star} - \myVec{\theta}\|^2.
\end{equation*}
\item \label{itm:BoundLoss}(Bounded loss) The loss $\ell$ is bounded by $B>0$. 
\end{enumerate}

We first bound the change in the risk induced by the variation in the data-generating distribution. 
Fix ${\myVec{\theta}}$, and define
\[
g_{\myVec{\theta}}(x,y) \triangleq \ell\!\left(f_{\myVec{\theta}}(x),y\right).
\]
By Assumption~\ref{itm:BoundLoss}, $g_{\myVec{\theta}}$ is uniformly bounded, namely 
$|g_{\myVec{\theta}}(x,y)| \le B$ for all $(x,y)$. It is a standard property of total variation distance that, for any bounded measurable function $g$,
\[
\left|\mathbb{E}_{\mathbb{P}_t}[g(x,y)] - \mathbb{E}_{\mathbb{P}_{t+1}}[g(x,y)]\right|
\le 2 \|g\|_\infty \, D_{\rm TV}(\mathbb{P}_t;\mathbb{P}_{t+1}).
\]
Applying this with $g=g_{\myVec{\theta}}$ yields
\begin{align}
|f_t({\myVec{\theta}})-f_{t+1}({\myVec{\theta}})|
&=
\left|
\mathbb{E}_{(x,y)\sim \mathbb{P}_t}[g_{\myVec{\theta}}(x,y)]
-
\mathbb{E}_{(x,y)\sim \mathbb{P}_{t+1}}[g_{\myVec{\theta}}(x,y)]
\right| \notag\\
&\le 2B\,\cdot D_{\rm TV}(\mathbb{P}_t;\mathbb{P}_{t+1}).
\label{eq:uniform_risk_bound}
\end{align} 

Next, by Assumption~\ref{itm:Convexity} and the optimality of ${\myVec{\theta}}_t^\star$ for $f_t$, we have
\begin{equation}
f_t({\myVec{\theta}}_{t+1}^\star)
\geq
f_t({\myVec{\theta}}_t^\star)
+
\frac{\mu}{2}\|{\myVec{\theta}}_{t+1}^\star-{\myVec{\theta}}_t^\star\|_2^2,
\label{eq:strong_conv_step}
\end{equation}
where we used that $\nabla f_t({\myVec{\theta}}_t^\star)=0$.

On the other hand, since ${\myVec{\theta}}_{t+1}^\star$ minimizes $f_{t+1}$,
\[
f_{t+1}({\myVec{\theta}}_{t+1}^\star)\leq f_{t+1}({\myVec{\theta}}_t^\star).
\]
Combining this inequality with \eqref{eq:uniform_risk_bound} yields
\begin{align}
f_t({\myVec{\theta}}_{t+1}^\star)
&\leq f_{t+1}({\myVec{\theta}}_{t+1}^\star)+2B\cdot D_{\rm TV}(\mathbb{P}_t;\mathbb{P}_{t+1}) \notag\\
&\leq f_{t+1}({\myVec{\theta}}_t^\star)+2B\cdot D_{\rm TV}(\mathbb{P}_t;\mathbb{P}_{t+1}) \notag\\
&\leq f_t({\myVec{\theta}}_t^\star)+4B\cdot D_{\rm TV}(\mathbb{P}_t;\mathbb{P}_{t+1}).
\label{eq:risk_compare}
\end{align}
Substituting \eqref{eq:risk_compare} into \eqref{eq:strong_conv_step} gives
\[
f_t({\myVec{\theta}}_t^\star)+\frac{\mu}{2}\|{\myVec{\theta}}_{t+1}^\star-{\myVec{\theta}}_t^\star\|_2^2
\leq
f_t({\myVec{\theta}}_t^\star)+4B\cdot D_{\rm TV}(\mathbb{P}_t;\mathbb{P}_{t+1}),
\]
and therefore
\begin{equation}
\|{\myVec{\theta}}_{t+1}^\star-{\myVec{\theta}}_t^\star\|_2
\leq
2\sqrt{\frac{2B\cdot D_{\rm TV}(\mathbb{P}_t;\mathbb{P}_{t+1})}{\mu}}.
\label{eq:theta_drift_bound}
\end{equation}

Finally, by the definition of the compressed representations,
\[
z_{t+1}-z_t = \mathcal{C}({\myVec{\theta}}_{t+1}^\star-{\myVec{\theta}}_t^\star).
\]
Hence, by continuity of the operator,
\[
\|z_{t+1}-z_t\|_2
=
\|\mathcal{C}({\myVec{\theta}}_{t+1}^\star-{\myVec{\theta}}_t^\star)\|_2
\leq
\|A\|_2\,\|{\myVec{\theta}}_{t+1}^\star-{\myVec{\theta}}_t^\star\|_2.
\]
Combining this with \eqref{eq:theta_drift_bound}, we obtain
\[
\|z_{t+1}-z_t\|_2
\leq
2L\sqrt{\frac{2B\cdot D_{\rm TV}(\mathbb{P}_t;\mathbb{P}_{t+1})}{\mu}},
\]
which proves the claim.

\subsubsection{Proof of Lemma~\ref{lem:Orthogonality}}
\label{app:LemmOrtProof}
\begin{proof}
We prove the lemma in two steps. We first show that, under the stated expressivity assumption, the parameter vector realizing the true posterior probabilities is an optimizer of the instantaneous risk under the cross-entropy loss. We then show that the corresponding observation error is zero-mean and uncorrelated with the model output.

\paragraph{Step 1.}
Consider the standard cross-entropy loss
\begin{equation}
\ell\big(f_{\myVec{\theta}}(x),y\big)
=
-\sum_{c=1}^{C}\mathbf{1}\{y=c\}\log \big[f_{\myVec{\theta}}(x)\big]_c,
\label{eq:ce_loss_appendix}
\end{equation}
where $\big[f_{\myVec{\theta}}(x)\big]_c$ denotes the $c$th component of the predictive probability vector $f_{\myVec{\theta}}(x)\in\Delta^C$. Substituting \eqref{eq:ce_loss_appendix} into the instantaneous risk \eqref{eq:instantaneous_risk}, we obtain
\begin{align}
\mathcal{L}_t(\myVec{\theta})
&=
\mathbb{E}_{(x,y)\sim\mathbb{P}_t}
\left[
-\sum_{c=1}^{C}\mathbf{1}\{y=c\}\log \big[f_{\myVec{\theta}}(x)\big]_c
\right] \notag\\
&=
\mathbb{E}_{x\sim\mathbb{P}_t}
\left[
\mathbb{E}_{y|x\sim\mathbb{P}_t}
\left[
-\sum_{c=1}^{C}\mathbf{1}\{y=c\}\log \big[f_{\myVec{\theta}}(x)\big]_c
\,\middle|\, x
\right]
\right] \notag\\
&=
\mathbb{E}_{x\sim\mathbb{P}_t}
\left[
-\sum_{c=1}^{C}\mathbb{P}_t(y=c|x)\log \big[f_{\myVec{\theta}}(x)\big]_c
\right].
\label{eq:risk_ce_conditional}
\end{align}
Now let
\[
\myVec{p}_t(x)
:=
\big[\mathbb{P}_t(y=1|x),\ldots,\mathbb{P}_t(y=C|x)\big]\in\Delta^C.
\]
By assumption, there exists $\myVec{\theta}_t\in\mathbb{R}^d$ such that $f_{\myVec{\theta}_t}(x)=\myVec{p}_t(x)$ for all  $x\in{\rm supp}(\mathbb{P}_t)$. 
For any fixed $x$, the inner term in \eqref{eq:risk_ce_conditional} is the cross-entropy between the true conditional distribution $\myVec{p}_t(x)$ and the model prediction $f_{\myVec{\theta}}(x)$. Using the identity
\begin{equation}
-\sum_{c=1}^{C} p_c \log q_c
=
-\sum_{c=1}^{C} p_c \log p_c
+
\sum_{c=1}^{C} p_c \log \frac{p_c}{q_c},
\label{eq:cross_entropy_kl_identity}
\end{equation}
valid for any probability vectors $\myVec{p},\myVec{q}\in\Delta^C$, we may rewrite \eqref{eq:risk_ce_conditional} as
\begin{align}
\mathcal{L}_t(\myVec{\theta})
&=
\mathbb{E}_{x\sim\mathbb{P}_t}
\left[
-\sum_{c=1}^{C}\mathbb{P}_t(y=c|x)\log \mathbb{P}_t(y=c|x)
\right] \notag\\
&\quad+
\mathbb{E}_{x\sim\mathbb{P}_t}
\left[
\sum_{c=1}^{C}\mathbb{P}_t(y=c|x)
\log
\frac{\mathbb{P}_t(y=c|x)}{\big[f_{\myVec{\theta}}(x)\big]_c}
\right].
\label{eq:risk_decomposition_kl}
\end{align}
The first term in \eqref{eq:risk_decomposition_kl} does not depend on $\myVec{\theta}$, while the second term is a conditional Kullback--Leibler divergence and is therefore nonnegative. Consequently,
\begin{equation}
\mathcal{L}_t(\myVec{\theta})
\geq
\mathbb{E}_{x\sim\mathbb{P}_t}
\left[
-\sum_{c=1}^{C}\mathbb{P}_t(y=c|x)\log \mathbb{P}_t(y=c|x)
\right],
\label{eq:risk_lower_bound}
\end{equation}
with equality if and only if
$
f_{\myVec{\theta}}(x)
\equiv 
\big[\mathbb{P}_t(y=1|x),\ldots,\mathbb{P}_t(y=C|x)\big]
$
Hence, the parameter vector $\myVec{\theta}_t$ minimizing \eqref{eq:risk_lower_bound} 
is an optimizer of \eqref{eq:instantaneous_risk}, namely
\[
\myVec{\theta}_t \in \arg\min_{\myVec{\theta}} \mathcal{L}_t(\myVec{\theta}).
\]
By the definition of the instantaneous optimum in \eqref{eq:instantaneous_optimum}, we may therefore identify this optimizer with $\myVec{\theta}_t^\star$. Since, by construction of the latent representation, $\mathcal{G}(\myVec{z}_t)=\myVec{\theta}_t^\star$, it follows that the observation model \eqref{eq:observation_model} can be written as
\begin{equation}
{\rm OneHot}(y_t)
=
f_{\myVec{\theta}_t^\star}(x_t)+\myVec{w}_t.
\label{eq:obs_model_optimal_appendix}
\end{equation}

\paragraph{Step 2.}
Define the observation error as
\begin{equation}
\myVec{w}_t
:=
{\rm OneHot}(y_t)-f_{\myVec{\theta}_t^\star}(x_t).
\label{eq:error_definition_appendix}
\end{equation}
Since $f_{\myVec{\theta}_t^\star}(x)$ realizes the true posterior probabilities, we have
\begin{equation}
f_{\myVec{\theta}_t^\star}(x)
=
\mathbb{E}_{y|x\sim\mathbb{P}_t}\big[{\rm OneHot}(y)\mid x\big].
\label{eq:posterior_conditional_mean}
\end{equation}
Therefore,
\begin{align}
\mathbb{E}\big[\myVec{w}_t \mid x_t\big]
&=
\mathbb{E}\big[{\rm OneHot}(y_t)-f_{\myVec{\theta}_t^\star}(x_t)\mid x_t\big] \notag\\
&=
\mathbb{E}\big[{\rm OneHot}(y_t)\mid x_t\big]-f_{\myVec{\theta}_t^\star}(x_t) \notag\\
&=
\myVec{0},
\label{eq:conditional_zero_mean}
\end{align}
where the last equality follows from \eqref{eq:posterior_conditional_mean}. Taking expectation over $x_t$ yields
\begin{equation}
\mathbb{E}[\myVec{w}_t]=\myVec{0},
\label{eq:unconditional_zero_mean}
\end{equation}
and thus $\myVec{w}_t$ is zero-mean.

Moreover, since $f_{\myVec{\theta}_t^\star}(x_t)$ is a deterministic function of $x_t$, \eqref{eq:conditional_zero_mean} implies the orthogonality relation
\begin{align}
\mathbb{E}\!\left[\myVec{w}_t f_{\myVec{\theta}_t^\star}(x_t)^\top\right]
&=
\mathbb{E}\!\left[
\mathbb{E}\!\left[\myVec{w}_t f_{\myVec{\theta}_t^\star}(x_t)^\top \mid x_t\right]
\right] \notag\\
&=
\mathbb{E}\!\left[
f_{\myVec{\theta}_t^\star}(x_t)^\top
\mathbb{E}\!\left[\myVec{w}_t \mid x_t\right]
\right] \notag\\
&=
\myVec{0}.
\label{eq:orthogonality_appendix}
\end{align}
Hence, the estimation error is orthogonal to the model output. 
\end{proof}



\subsubsection{Proof of Lemma~\ref{lem:BONGtoKF}}
\label{app:bong_proof}
The lemma extends the corresponding characterization in \cite{jones2024bayesian} to Bayesian learning in latent space. Accordingly, the proof follows similar arguments as in \cite[Appendix E]{jones2024bayesian}. 
We begin by writing the ELBO objective associated with the Gaussian variational posterior based on the Bayesian learning rule of~\cite{khan2023bayesian}. Under a Gaussian modeling of the parameters $\myVec{z}_t$ as $\myVec{z}_t\sim \mathcal{N}(\hat{\myVec{z}}_t, \myMat{\Sigma}_t) := q_{\myVec{\phi}_t }$ where   $\myVec{\phi}_t := \{\myMat{\Sigma}^{-1}_{t}\myVec{\hat{z}}_{t},-\frac{1}{2}\myMat{\Sigma}^{-1}_{t}\}$ encapsulates the distribution parameters of  $\myVec{z}_t$. The online ELBO can be written as
\begin{equation}
\label{eq:ELBO}
\mathcal{L}(\myVec{\phi}_t)
=
- \mathbb{E}_{\myVec{z}_t \sim q_{\myVec{\phi}_t}}
\left[
\log p(y_t \mid {x}_t, \mathcal{G}(\myVec{z}_t)
\right]
+
D_{\mathrm{KL}}\!\left(q_{\myVec{\phi}_t}\,\|\,q_{\myVec{\phi}_{t\mid t-1}}\right),
\end{equation}
where we write $ p(y \mid {x}, \myVec{\theta})$ as the element of the model output $f_{\myVec{\theta}}(x)$ corresponding to $y \in \{1,\ldots,C\}$.

For this Gaussian exponential-family representation, the corresponding
dual, or expectation, parameters are
\begin{align}
\label{eq:Natural Parameter}
\myVec{\rho}^{(1)}_{t}
&=
\hat{\myVec{z}}_{t},
&
\myMat{\rho}^{(2)}_{t}
&=
\hat{\myVec{z}}_{t}\hat{\myVec{z}}_{t}^{\top}
+
\myMat{\Sigma}_{t}.
\end{align}
Equivalently, the mean and covariance can be recovered from the natural
or dual parameters as
\begin{align}
\hat{\myVec{z}}_{t}
&=
-\frac{1}{2}
\left(\myMat{\phi}^{(2)}_{t}\right)^{-1}
\myVec{\phi}^{(1)}_{t}
=
\myVec{\rho}^{(1)}_{t},
\\
\myMat{\Sigma}_{t}
&=
-\frac{1}{2}
\left(\myMat{\phi}^{(2)}_{t}\right)^{-1}
=
\myMat{\rho}^{(2)}_{t}
-
\myVec{\rho}^{(1)}_{t}
\big(\myVec{\rho}^{(1)}_{t}\big)^{\top}.
\end{align}

Starting from the prior predictive natural parameter
$\myVec{\phi}_{t|t-1}$, a single natural-gradient step on
\eqref{eq:ELBO} yields
\begin{equation}
\label{eq:First step NG}
\myVec{\phi}_t
=
\myVec{\phi}_{t|t-1}
+
\mathbf{F}^{-1}_{\myVec{\phi}_{t|t-1}}
\nabla_{\myVec{\phi}_{t|t-1}}
\,
\mathbb{E}_{\myVec{z}_t \sim q_{\myVec{\phi}_{t|t-1}}}
\left[
\log p({y}_t \mid {x}_t, \mathcal{G}(\myVec{z}_t))
\right],
\end{equation}
where $\mathbf{F}_{\myVec{\phi}}$ is the Fisher information matrix with
respect to the natural parameters $\myVec{\phi}$.
Since the Gaussian variational posterior belongs to the exponential
family, the natural gradient with respect to the natural parameters
coincides with the ordinary gradient with respect to the corresponding
dual parameters $\myVec{\rho}$. Therefore, \eqref{eq:First step NG} can be
equivalently written as 
\begin{equation}
\label{eq:First step NG dual}
\myVec{\phi}_t
=
\myVec{\phi}_{t|t-1}
+
\nabla_{\myVec{\rho}_{t|t-1}}
\,
\mathbb{E}_{\myVec{z}_t \sim q_{\myVec{\phi}_{t|t-1}}}
\left[
\log p({y}_t \mid {x}_t,\mathcal{G}(\myVec{z}_t))
\right].
\end{equation}

Using the relations between the Gaussian natural, dual, and moment parameters, the chain rule yields, for any scalar function $\ell$,
\begin{equation}
\label{eq:rho_chain_rule}
\begin{aligned}
\nabla_{\myVec{\rho}^{(1)}_{t}} \ell
&=
\frac{\partial \myVec{\hat{z}}_{t}}{\partial \myVec{\rho}^{(1)}_{t}}
\nabla_{\myVec{\hat{z}}_{t}} \ell
+
\frac{\partial \myMat{\Sigma}_{t}}{\partial \myVec{\rho}^{(1)}_{t}}
\nabla_{\myMat{\Sigma}_{t}} \ell
=
\nabla_{\myVec{\hat{z}}_{t}} \ell
-
2\big(\nabla_{\myMat{\Sigma}_{t}} \ell\big)\myVec{\hat{z}}_{t},
\\[0.5em]
\nabla_{\myMat{\rho}^{(2)}_{t}} \ell
&=
\frac{\partial \myVec{\hat{z}}_{t}}{\partial \myMat{\rho}^{(2)}_{t}}
\nabla_{\myVec{\hat{z}}_{t}} \ell
+
\frac{\partial \myMat{\Sigma}_{t}}{\partial \myMat{\rho}^{(2)}_{t}}
\nabla_{\myMat{\Sigma}_{t}} \ell
=
\nabla_{\myMat{\Sigma}_{t}} \ell.
\end{aligned}
\end{equation}
Substituting 
$\ell =
\mathbb{E}_{\myVec{z}_t \sim q_{\myVec{\phi}_t}}
\!\left[
\log p\big(y_t \mid x_t, \mathcal{G}(\myVec{z}_t)\big)
\right]$
into the chain-rule identities, and using Bonnet's and Price's identities~\cite{bonnet1964,price1958}, we obtain
%
\begin{align}
\nabla_{\hat{\myVec{z}}_{t}}
\mathbb{E}_{\myVec{z}_t \sim q_{\myVec{\phi}_t}}
\!\left[
\log p\big(y_t \mid x_t,\mathcal{G}(\myVec{z}_t)\big)
\right]
&=
\mathbb{E}_{\myVec{z}_t \sim q_{\myVec{\phi}_t}}
\!\left[
\nabla_{\myVec{z}_t}
\log p\big(y_t \mid x_t,\mathcal{G}(\myVec{z}_t)\big)
\right],
\\
\nabla_{\myMat{\Sigma}_{t}}
\mathbb{E}_{\myVec{z}_t \sim q_{\myVec{\phi}_t}}
\!\left[
\log p\big(y_t \mid x_t,\mathcal{G}(\myVec{z}_t)\big)
\right]
&=
\frac{1}{2}
\mathbb{E}_{\myVec{z}_t \sim q_{\myVec{\phi}_t}}
\!\left[
\nabla_{\myVec{z}_t}^{2}
\log p\big(y_t \mid x_t,\mathcal{G}(\myVec{z}_t)\big)
\right].
\end{align}

Substituting the above identities into \eqref{eq:rho_chain_rule} yields
\begin{align}
\nabla_{\myVec{\rho}^{(1)}_{t}}
\mathbb{E}_{\myVec{z}_t \sim q_{\myVec{\phi}_t}}
\!\left[
\log p\big(y_t \mid x_t,\mathcal{G}(\myVec{z}_t)\big)
\right]
&=
\nabla_{\hat{\myVec{z}}_{t}}
\mathbb{E}_{\myVec{z}_t \sim q_{\myVec{\phi}_t}}
\!\left[
\log p\big(y_t \mid x_t,\mathcal{G}(\myVec{z}_t)\big)
\right]
\notag\\
&\quad
-
2
\left(
\nabla_{\myMat{\Sigma}_{t}}
\mathbb{E}_{\myVec{z}_t \sim q_{\myVec{\phi}_t}}
\!\left[
\log p\big(y_t \mid x_t,\mathcal{G}(\myVec{z}_t)\big)
\right]
\right)
\hat{\myVec{z}}_{t}
\notag\\
&=
\mathbb{E}_{\myVec{z}_t \sim q_{\myVec{\phi}_t}}
\!\left[
\nabla_{\myVec{z}_t}
\log p\big(y_t \mid x_t,\mathcal{G}(\myVec{z}_t)\big)
\right]
\notag\\
&\quad
-
\mathbb{E}_{\myVec{z}_t \sim q_{\myVec{\phi}_t}}
\!\left[
\nabla_{\myVec{z}_t}^{2}
\log p\big(y_t \mid x_t,\mathcal{G}(\myVec{z}_t)\big)
\right]
\hat{\myVec{z}}_{t},
\label{eq:rho1_expected_loglik_explicit}
\\[0.5em]
\nabla_{\myMat{\rho}^{(2)}_{t}}
\mathbb{E}_{\myVec{z}_t \sim q_{\myVec{\phi}_t}}
\!\left[
\log p\big(y_t \mid x_t,\mathcal{G}(\myVec{z}_t)\big)
\right]
&=
\nabla_{\myMat{\Sigma}_{t}}
\mathbb{E}_{\myVec{z}_t \sim q_{\myVec{\phi}_t}}
\!\left[
\log p\big(y_t \mid x_t,\mathcal{G}(\myVec{z}_t)\big)
\right]
\notag\\
&=
\frac{1}{2}
\mathbb{E}_{\myVec{z}_t \sim q_{\myVec{\phi}_t}}
\!\left[
\nabla_{\myVec{z}_t}^{2}
\log p\big(y_t \mid x_t,\mathcal{G}(\myVec{z}_t)\big)
\right].
\label{eq:rho2_expected_loglik_explicit}
\end{align}
Equivalently, in terms of the posterior mean and covariance, the update becomes
\begin{equation}
\label{eq:Basic SSM}
\begin{aligned}
\hat{\myVec{z}}_t
&=
\hat{\myVec{z}}_{t|t-1}
+
\myMat{\Sigma}_t
\mathbb{E}_{\myVec{z}_t \sim q_{\myVec{\phi}_{t|t-1}}}
\!\left[
\nabla_{\myVec{z}_t}
\log p\big(y_t \mid x_t,\mathcal{G}(\myVec{z}_t)\big)
\right],
\\[0.5em]
\myMat{\Sigma}_{t}^{-1}
&=
\myMat{\Sigma}_{t|t-1}^{-1}
-
\mathbb{E}_{\myVec{z}_t \sim q_{\myVec{\phi}_{t|t-1}}}
\!\left[
\nabla_{\myVec{z}_t}^{2}
\log p\big(y_t \mid x_t,\mathcal{G}(\myVec{z}_t)\big)
\right].
\end{aligned}
\end{equation}


Next, recall that the state-space model   assumes Gaussian process and observation noises, namely, \(\myVec{v}_t \sim \mathcal{N}(\myVec{0},\myMat{Q})\) and \(\myVec{w}_t \sim \mathcal{N}(\myVec{0},\myMat{R})\). Consequently, the prior predictive distribution of the latent state is Gaussian and can be written as
\[
q_{\myVec{\phi}_{t\mid t-1}}\sim
\mathcal{N}\big(\hat{\myVec{z}}_{t\mid t-1},\myMat{\Sigma}_{t\mid t-1}\big).
\]
Moreover, under the Gaussian observation model 
\[
 p({y}_t \mid {x}_t,\mathcal{G}(\myVec{z}_t))
\sim
\mathcal{N}\!\left(
f_{\mathcal{G}(\myVec{z}_t)}(\myVec{x}_t),\myMat{R}
\right),
\]
we replace the nonlinear predictor with its first-order Taylor approximation around \(\hat{\myVec{z}}_{t\mid t-1}\), namely,
\[
f_{\mathcal{G}(\myVec{z}_t)}(\myVec{x}_t)
=
f_{\mathcal{G}(\hat{\myVec{z}}_{t\mid t-1})}(\myVec{x}_t)
+
\myMat{H}_t\big(\myVec{z}_t-\hat{\myVec{z}}_{t\mid t-1}\big).
\]
Under these assumptions, the expected gradient and Hessian admit the closed-form expressions
\begin{equation}
\label{eq:Linearization of Expectation}
\scalebox{0.98}{$
\begin{aligned}
\mathbb{E}_{\myVec{z}_t \sim q_{\myVec{\phi}_{t\mid t-1}}}
\!\left[
\nabla_{\myVec{z}_t}
\log p_t(\myVec{y}_t \mid \myVec{z}_t)
\right]
&=
\mathbb{E}_{\myVec{z}_t \sim q_{\myVec{\phi}_{t\mid t-1}}}
\!\left[
\myMat{H}_t^{\top}\myMat{R}^{-1}
\Big(
\text{OneHot}(y_t)-\hat{\myVec{y}}_t-\myMat{H}_t(\myVec{z}_t-\hat{\myVec{z}}_{t\mid t-1})
\Big)
\right]
\\
&=
\myMat{H}_t^{\top}\myMat{R}^{-1}
\big(
\text{OneHot}(y_t)-\hat{\myVec{y}}_t
\big),
\\[0.5em]
\mathbb{E}_{\myVec{z}_t \sim q_{\myVec{\phi}_{t\mid t-1}}}
\!\left[
\nabla_{\myVec{z}_t}^{2}
\log p_t(\myVec{y}_t \mid \myVec{z}_t)
\right]
&=
\mathbb{E}_{\myVec{z}_t \sim q_{\myVec{\phi}_{t\mid t-1}}}
\!\left[
-\myMat{H}_t^{\top}\myMat{R}^{-1}\myMat{H}_t
\right]
\\
&=
-\myMat{H}_t^{\top}\myMat{R}^{-1}\myMat{H}_t.
\end{aligned}
$}
\end{equation}

Applying \eqref{eq:Linearization of Expectation} to \eqref{eq:Basic SSM}, and rearranging the resulting expressions, yields
\begin{equation*}
\begin{aligned}
\myMat{K}_t
&=
\myMat{\Sigma}_{t|t-1}\myMat{H}_t^{\top}
\big(
\myMat{H}_t\myMat{\Sigma}_{t|t-1}\myMat{H}_t^{\top}+\myMat{R}
\big)^{-1},
\\[0.5em]
\hat{\myVec{z}}_t
&=
\hat{\myVec{z}}_{t|t-1}
+
\myMat{K}_t
\Big(
\text{OneHot}(y_t)-f_{\mathcal{G}(\hat{\myVec{z}}_{t|t-1})}(\myVec{x}_t)
\Big),
\\[0.5em]
\myMat{\Sigma}_t
&=
\myMat{\Sigma}_{t|t-1}
-
\myMat{K}_t\myMat{H}_t\myMat{\Sigma}_{t|t-1}.
\end{aligned}
\end{equation*}

Hence, under the Gaussian noise assumption and the first-order linearization of the observation model, minimizing the ELBO via a single natural-gradient step is equivalent to a Kalman-style single-step adaptation update.


\section{Implementation Details of the Experiments in Section \ref{sec:experiments}}
\label{app:exp_details}


\subsection{Non-stationary Wireless Receiver Adaptation}
\label{sec:settingWireless}

\paragraph{Time-Varying MIMO Detection Scenario.}
We consider an uplink multiuser MIMO system with \(K\) single-antenna transmitters and \(N_r\) receive antennas. At each time frame, the users transmit a vector of information symbols, and the receiver observes a noisy superposition governed by a time-varying wireless channel. Specifically, for a transmitted QPSK symbol vector \(\myVec{s}_t \in \mathbb{R}^{K}\), which modulates the $2K$ bits, i.e., $C = 2^{2K}$, the received signal is modeled as
\begin{equation}
    \label{eq:MIMO}
    \myVec{x}_t = \myMat{H}_t \myVec{s}_t + \myVec{w}_t,
    \qquad
    \myVec{w}_t \sim \mathcal{N}(\myVec{0},\myMat{P}_t),
\end{equation}
where \(\myMat{H}_t \in \mathbb{R}^{N_r \times K}\) is the channel matrix and \(\myMat{P}_t\) denotes the noise covariance. The variation of \(\myMat{H}_t\) across time frames makes the MIMO detection rule time dependent, thereby giving rise to a non-stationary online detection setting.

The receiver architecture is based on DeepSIC~\cite{shlezinger2019deepsic}, a learned implementation of iterative soft interference cancellation for multiuser detection. We evaluate this receiver under both linear and nonlinear observation models. In the linear case, the received signal follows \eqref{eq:MIMO}. In the nonlinear case, we apply a component-wise \(\tanh\) distortion to the received signal, resulting in a more challenging setting that departs from the nominal linear MIMO model and tests the robustness of the adaptation mechanism.

\paragraph{Protocol and Channel Generation.}
We instantiate the system with \(K=3\) users and \(N_r=5\) receive antennas. Time-varying channel trajectories are generated using the QuaDRiGa simulator under the 3GPP Indoor Office LOS scenario, yielding temporally correlated channel realizations that emulate realistic propagation dynamics. Each evaluation sequence contains \(150\) consecutive frames. The channel realization remains fixed over the duration of each frame and varies across frames, thereby inducing a sequence of temporally correlated detection problems.

The online protocol consists of two phases. The first four frames serve as an initial synchronization stage, during which the receiver is provided with \(64\) labeled pilot samples per frame. These samples are used to perform consecutive online updates and initialize the adaptive receiver state. The remaining \(146\) frames constitute the tracking stage. In each tracking frame, only \(6\) labeled pilot samples are available for adaptation, corresponding to \(6\) online updates. The adapted receiver is then evaluated on \(1000\) additional QPSK symbol 
vectors generated under the same channel realization, i.e., \(2000\) information bits per user.
This setup reflects the operational structure of practical communication links: a short synchronization period enables reliable initialization, whereas subsequent operation must track channel variations using only limited pilot overhead. To avoid leakage between phases, all offline training is performed on channel trajectories disjoint from those used for online evaluation. During online operation, each method may update only from the labeled pilots revealed in the current stream.

The neural receiver is based on the DeepSIC  multi-user detection network \cite{shlezinger2019deepsic}. It is composed of four multi-layer perceptron modules, each containing one detector block
per transmitted user. Each block produces a soft estimate of its associated
user symbol from the received signal and the current soft estimates of the
interfering users; these estimates are propagated across layers to
progressively refine the multi-user decision. In our implementation, each
block is a compact two-layer neural network with a nonlinear activation and
\(868\) trainable parameters. Online adaptation is therefore performed at the
block-detector level, with the proposed method replacing direct full-space
block updates by updates in a learned low-dimensional latent representation
from which the block parameters are reconstructed.

\paragraph{Compared Methods.}
We compare \ac{name} against representative online adaptation baselines. As a gradient-based reference, we include \emph{Online GD}, which updates the full set of receiver parameters directly using the labeled pilots available at each frame  and perform five consecutive gradient updates per labeled sample. We further consider parameter-space \ac{ekf}-based learning, using  BONG~\cite{jones2024bayesian}, with full, diagonal~\cite{Chang2022on}, and \ac{dlr} covariance parameterizations~\cite{Chang2023low}. These baselines adapt directly in the original parameter space of the DeepSIC detector.

\paragraph{Implementation and Compute.}
All wireless receiver experiments were implemented in PyTorch and executed
on CPU using a 13th Gen Intel(R) Core(TM) i7-13620H processor with \(10\)
physical cores and \(16\) logical threads. All methods were evaluated on the
same hardware configuration and with the same evaluation pipeline to ensure
a fair comparison.

\subsection{Non-stationary Image Classification}
\label{sec:settingVision}

\paragraph{Gradual Distribution Shift Scenario.}
We consider a pretrained image classification model trained on a clean source dataset, such as CIFAR-10, CIFAR-100, or ImageNet. Let \(\{\tilde{\myVec{x}}_{t,i}\}_{i=1}^{V}\) denote the set of corrupted versions of the same underlying image at time step \(t\), where \(V\) is the number of corruption types under consideration. To model a gradually evolving test distribution, we define a mixing vector \(\myVec{\alpha}_t \in \mathbb{R}^{V}\), satisfying \(\sum_{i=1}^{V} \alpha_{t,i} = 1\), whose entries determine the contribution of each corruption type at time \(t\).

Initially, the mixing process is set such that all mass is assigned to the clean component, namely \(\myVec{\alpha}_0 = [1,0,\ldots,0]\). At each time step \(t\), the observed test sample is generated as a convex combination of the corrupted versions of the same image, weighted by the current mixing vector:
\begin{equation}
    \label{eq:tta_sample_generation}
    \myVec{x}_t = \sum_{i=1}^{V} \alpha_{t,i}\, \tilde{\myVec{x}}_{t,i}.
\end{equation}
Thus, as \(\myVec{\alpha}_t\) evolves over time according to a stochastic process, the induced test distribution undergoes a gradual and continuous shift. In the considered online adaptation setting, only a small subset of samples is labeled at each time step, enabling the adaptation methods to update using a limited amount of reliable supervision.

\paragraph{Protocol and Datasets.}
Concretely, each evaluation run consists of a stream of $T = 1000$ time steps (frames), each containing a batch of $100$ images, for a total of $10^5$ test samples per run. We draw corruptions from a pool of $V = 6$ types covering two qualitatively distinct regimes: three noise corruptions (Gaussian/shot/impulse on CIFAR; shot/impulse/fog on MNIST) and three blur/photometric corruptions (defocus, glass, motion on CIFAR; motion blur, glass blur, brightness on MNIST). The mixing vector \(\myVec{\alpha}_t\) evolves as a random walk on the simplex, so that consecutive frames induce similar but non-identical mixtures and the induced shift is both gradual and non-stationary. We evaluate on three standard corruption benchmarks, MNIST-C, CIFAR-10-C, and CIFAR-100-C, applying the corruptions at severity~$5$ on CIFAR-10-C and CIFAR-100-C. Each backbone is a standard image-classification architecture pretrained on the corresponding clean source dataset, and is kept fixed throughout the test-time stream: on MNIST-C we use a ResNet-18~\cite{he2016resnet} adapted to single-channel inputs, trained in-house on clean MNIST to $99.34\%$ test accuracy; on CIFAR-10-C we use a Wide Residual Network~\cite{zagoruyko2016wide} with $28$ layers and a widen factor of $10$, pretrained on clean CIFAR-10 with standard supervised training; and on CIFAR-100-C we use a $29$-layer ResNeXt~\cite{xie2017aggregated} pretrained on clean CIFAR-100 with AugMix~\cite{hendrycks2019augmix}. The CIFAR-100 backbone is obtained from RobustBench~\cite{croce2020robustbench}.

\paragraph{Compared Methods.}
We benchmark the proposed method against representative online adaptation strategies spanning non-adaptive, gradient-based, normalization-based, self-supervised, and Bayesian approaches. As a reference point, we include the \emph{Source} model, which applies the pretrained classifier without any test-time parameter updates. To capture supervised gradient-based adaptation, we consider an \emph{Online GD} baseline that updates the model directly using the labeled samples obtained on each new frame  and perform five consecutive gradient update per labeled sample. We also evaluate a family of normalization-based baselines that adapt the batch-normalization statistics at test time, including a standard BN adaptation rule~\cite{schneider2020corruptions,nado2020bn} (\emph{BN-Test}), an exponential-moving-average variant (\emph{BN-EMA}), and an interpolation-based variant that combines source and target statistics~\cite{mirza2022bn-dua,lim2023ttn} (\emph{BN-Alpha}). As a self-supervised alternative, we evaluate \emph{ROID}~\cite{marsden2024roid}. We further consider a Bayesian filtering baseline based on the Kalman-filter framework of~\cite{titsias2024kalman}, adapted in our setting to operate on the final fully-connected classification layer. Finally, we compare these methods with the proposed approach, which performs Bayesian online adaptation in a learned low-dimensional latent space (with dimension size $m=128$), while reconstructing the full parameter vector for prediction.
In all image-classification experiments, online adaptation is applied to the final fully-connected classification layer together with selected batch-normalization parameters throughout the backbone. To keep the Bayesian update tractable, we assume independence between the adapted parameter groups, corresponding to a block-diagonal covariance structure across groups. This type of block-structured approximation has also been used in scalable Kalman-style neural-network updates to avoid modeling all cross-parameter dependencies~\cite{duranmartin2025martingale}, thereby reducing complexity.

\paragraph{Implementation and Compute.}
All image adaptation experiments were implemented in PyTorch and executed on
an NVIDIA RTX 3060 Ti GPU. All methods were evaluated using the same hardware
configuration and experimental pipeline to ensure a fair comparison.

\section{Additional Experiments}
\label{app:additional}
 This appendix provides supplementary experimental studies that further analyze the behavior and robustness of \ac{name} beyond the main-paper results of Section~\ref{sec:experiments}. These experiments are designed to provide deeper insight into the properties of the proposed latent Bayesian adaptation framework, including its sensitivity to latent dimensionality, robustness to train-test mismatch, and dependence on the offline meta-learning stage. We report additional studies for both application domains considered in this work: wireless receiver adaptation and non-stationary image classification. Together, these results further validate the generality, robustness, and practical design trade-offs of the proposed approach.

\subsection{Additional Wireless Communication Experiments}
\label{app:wireless_ablation}
\subsubsection{Temporal Tracking Behavior}
\label{subsec:wireless_temporal_tracking}

We complement the BER-SNR results in Section~\ref{ssec:WirelessExp} by reporting the temporal tracking behavior over the online deployment horizon. Fig.~\ref{fig:wireless_temporal_tracking} shows the BER as a function of the snapshot index for the linear and nonlinear QuaDRiGa channels. The results indicate that \ac{name} maintains stable tracking throughout the channel evolution, supporting the aggregate gains reported in ~\ref{sec:experiments}.

\begin{figure}[t]
    \centering
    \begin{subfigure}[t]{0.49\linewidth}
        \centering
        \includegraphics[
            width=\linewidth
        ]{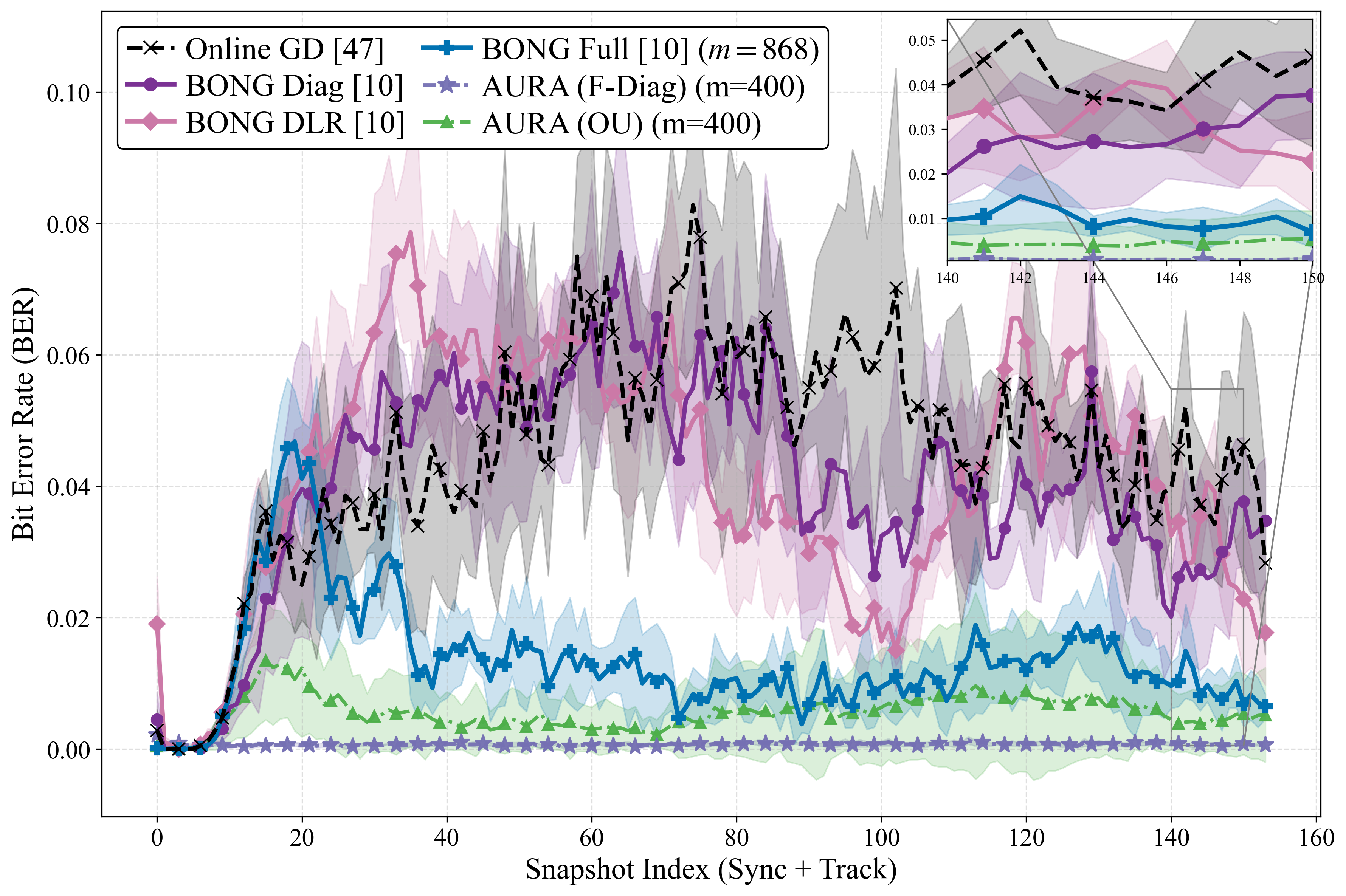}
        \caption{Linear channel.}
        \label{fig:wireless_temporal_tracking_linear}
    \end{subfigure}
    \hfill
    \begin{subfigure}[t]{0.49\linewidth}
        \centering
        \includegraphics[
            width=\linewidth
        ]{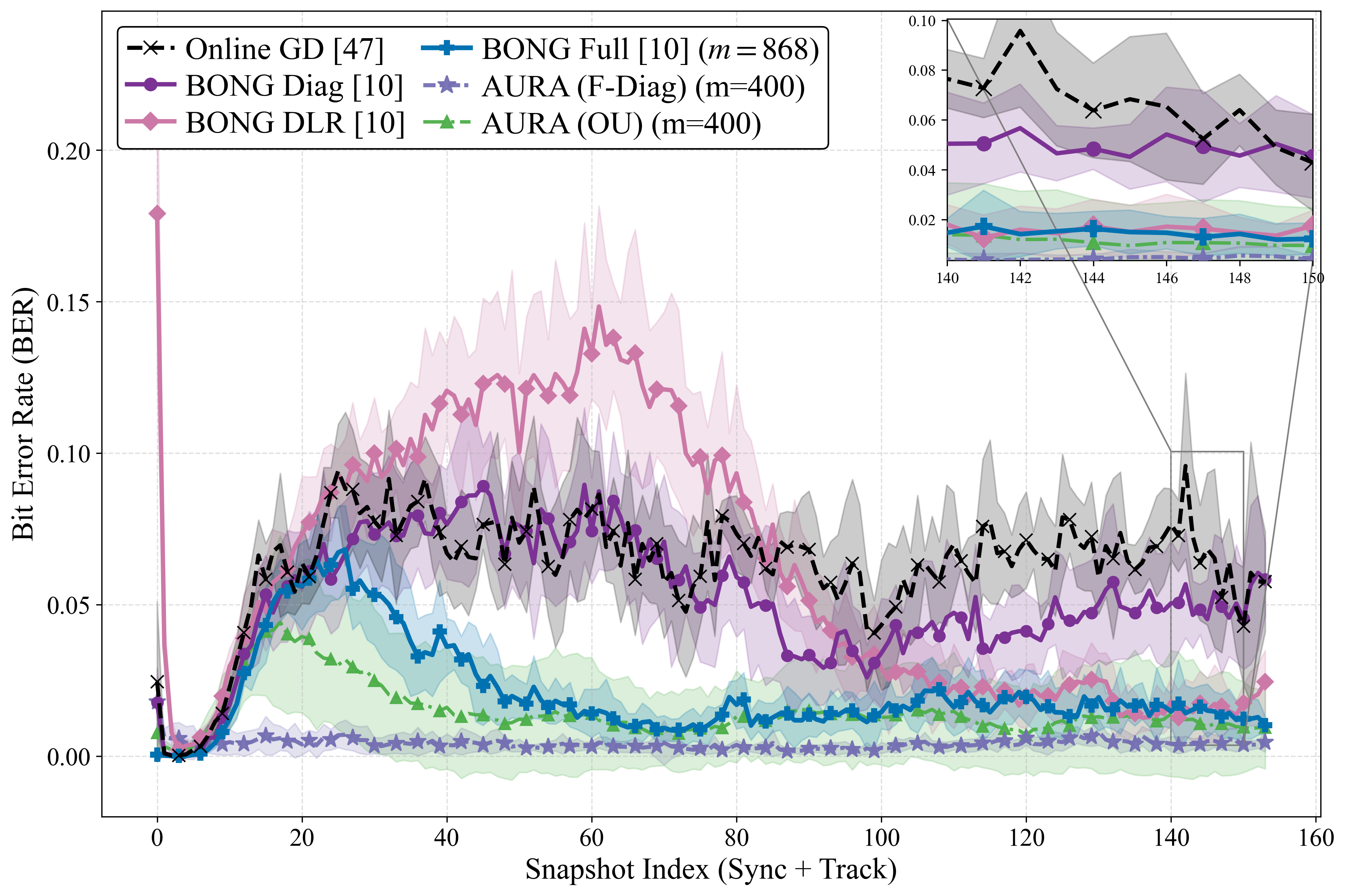}
        \caption{Nonlinear channel.}
        \label{fig:wireless_temporal_tracking_nonlinear}
    \end{subfigure}
    \caption{BER vs. Snapshot Index}
    \label{fig:wireless_temporal_tracking}
\end{figure}

    \subsubsection{Effect of Latent Compression and Dynamics Parameterization}

While Subsection~\ref{ssec:VisionExp} studies the effect of the latent dimension $m$ in the non-stationary image classification setting, we here provide a complementary analysis in the wireless receiver domain that additionally examines the interaction between latent compression and the choice of temporal dynamics model. Specifically, we assess how the dimensionality of the learned latent adaptation space affects performance under two alternative state evolution parameterizations: $(i)$ OU-based dynamics, corresponding to the scaled-identity transition model $\myVec{F}=\gamma\myMat{I}$, and $(ii)$ fully learned linear dynamics, where $\myVec{F}$ is meta-learned from data.
To this end, we vary the latent compression ratio used by \ac{name}, corresponding to different latent dimensions $m$, and evaluate the resulting online adaptation performance under the linear QuaDRiGa channel model. This experiment aims to characterize the trade-off between latent compression, adaptation fidelity, and temporal modeling flexibility.

\begin{figure}[t]
    \centering
    \begin{minipage}[t]{0.48\linewidth}
        \centering
        \includegraphics[
            width=\linewidth
        ]{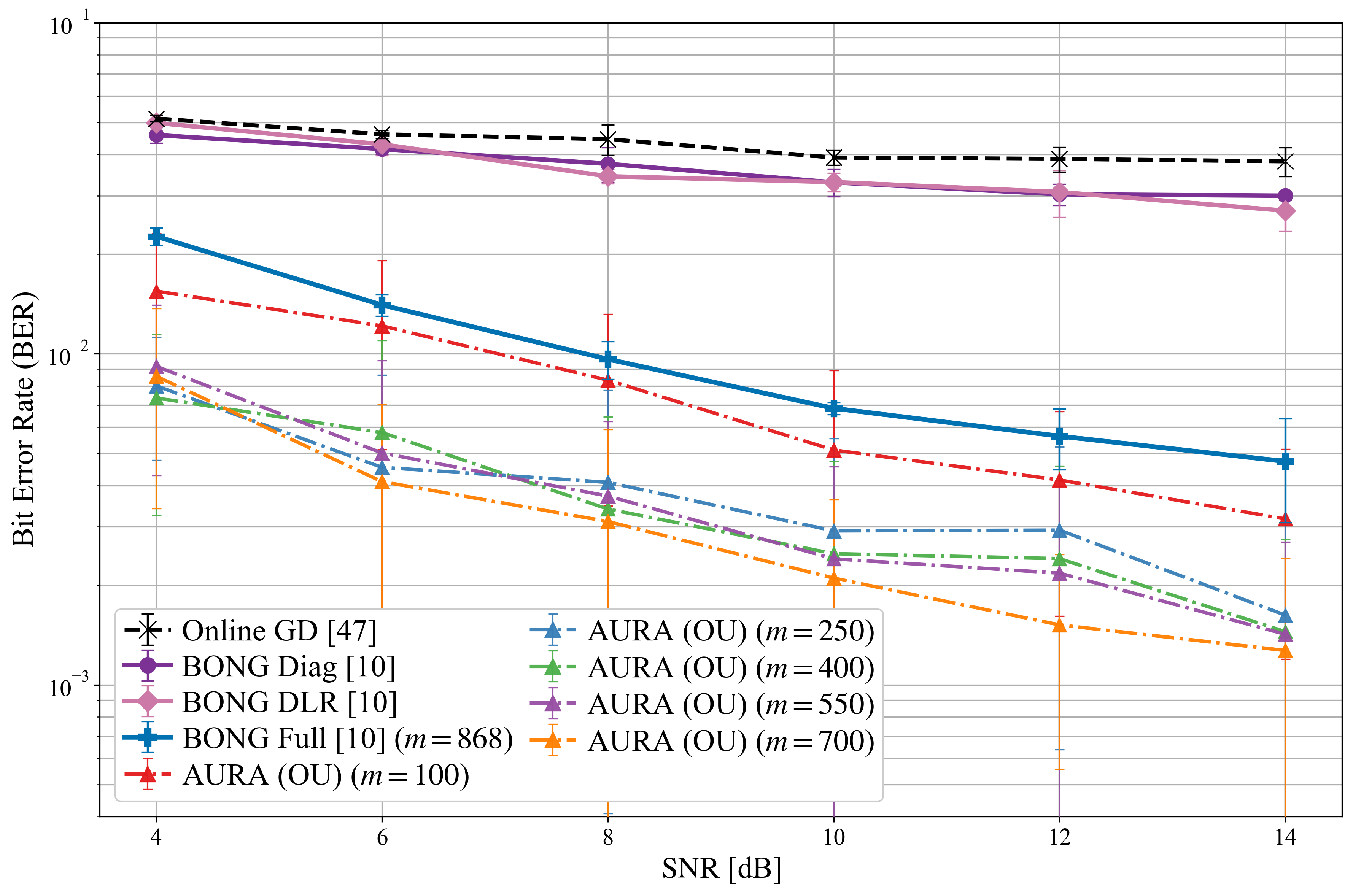}
        \caption{OU Ablation: BER vs. SNR.}
        \label{fig:ablation_ou_linear}
    \end{minipage}
    \hfill
    \begin{minipage}[t]{0.48\linewidth}
        \centering
        \includegraphics[
            width=\linewidth
        ]{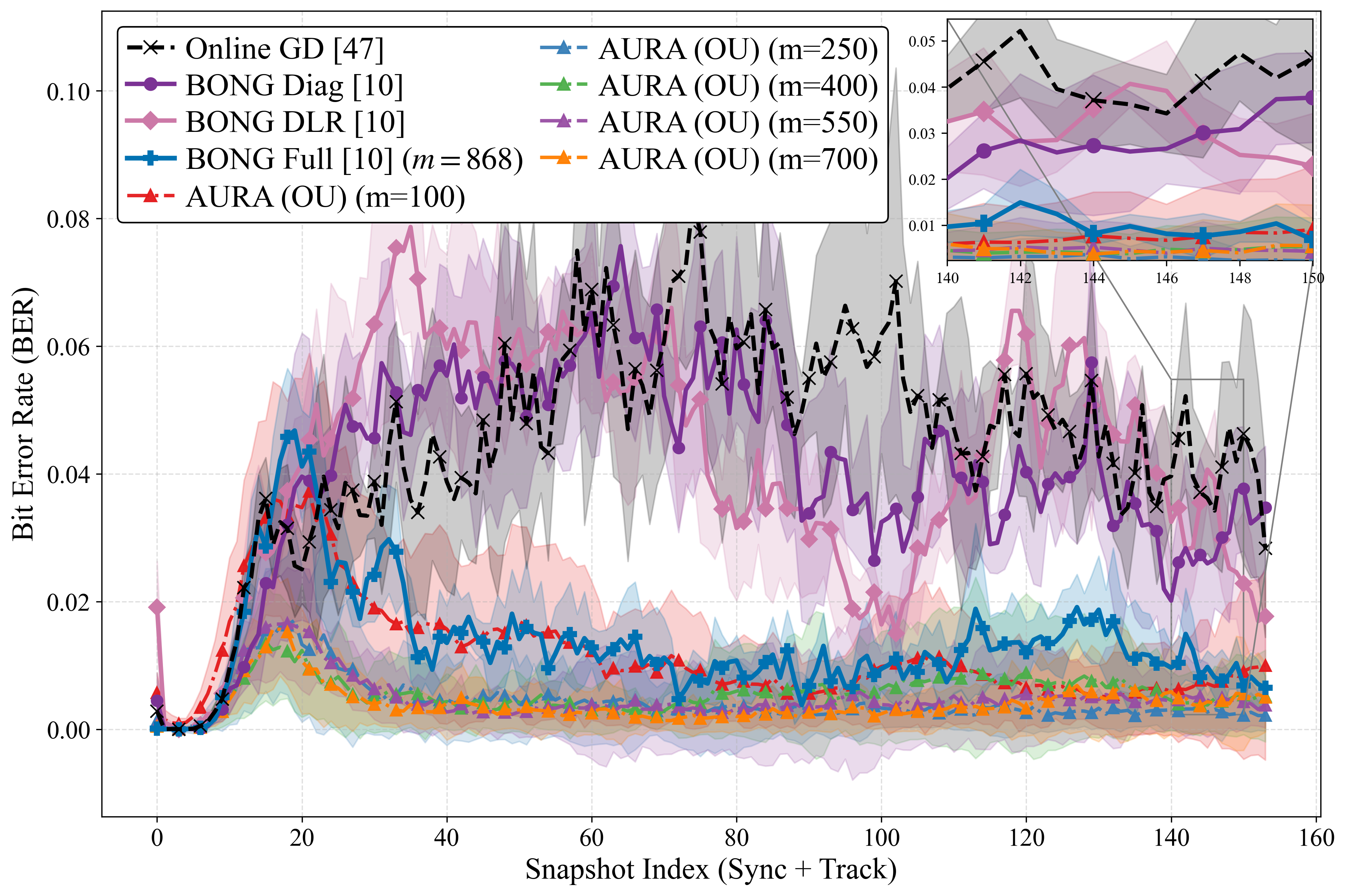}
        \caption{OU Ablation: BER vs. Snapshot Index.}
        \label{fig:ablation_ou_per_snapshot}
    \end{minipage}
\end{figure}

\begin{figure}[t]
    \centering
    \begin{minipage}[t]{0.48\linewidth}
        \centering
        \includegraphics[
            width=\linewidth
        ]{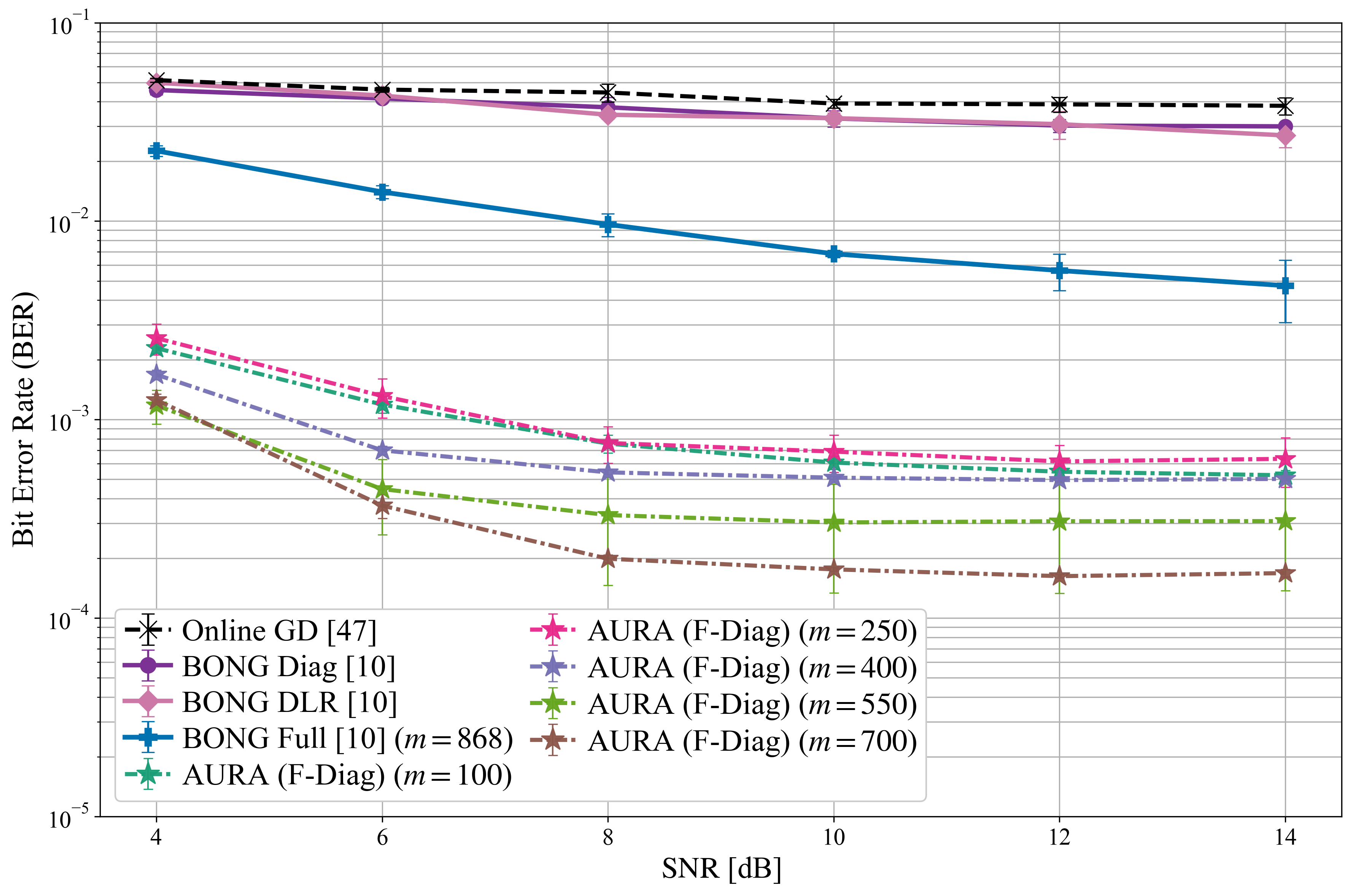}
        \caption{$F$-Diagonal Ablation: BER vs. SNR.}
        \label{fig:ablation_f_diag_linear}
    \end{minipage}
    \hfill
    \begin{minipage}[t]{0.48\linewidth}
        \centering
        \includegraphics[
            width=\linewidth
        ]{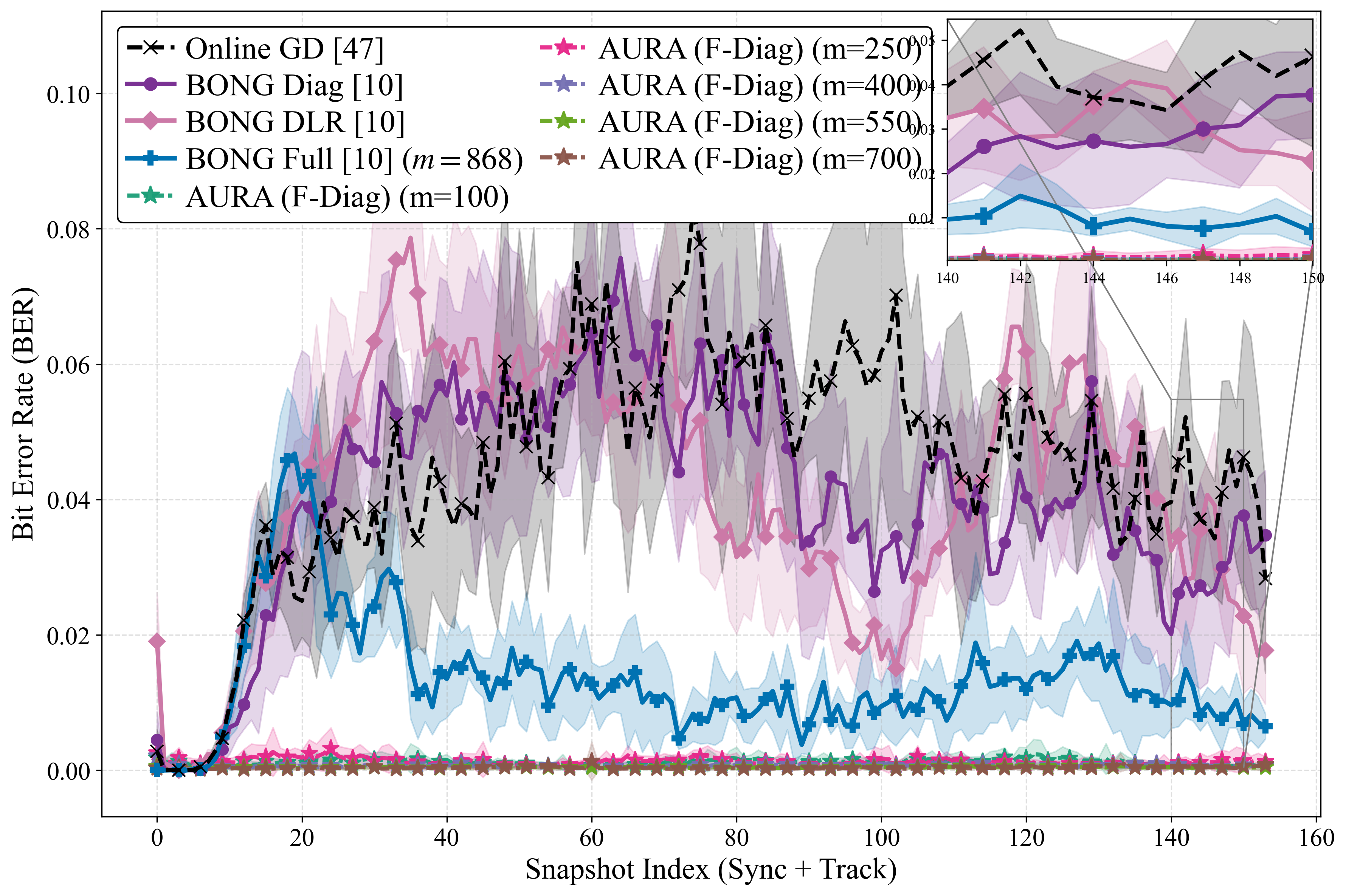}
        \caption{$F$-Diagonal Ablation: BER vs. Snapshot Index.}
        \label{fig:ablation_f_diag_per_snapshot}
    \end{minipage}
\end{figure}

The results in Figs.~\ref{fig:ablation_ou_linear}--\ref{fig:ablation_f_diag_per_snapshot} reveal several consistent trends. First, across all tested compression levels and for both temporal dynamics models, \ac{name} maintains a clear advantage over the corresponding full-parameter Bayesian baselines. This indicates that the gains of the proposed latent Bayesian adaptation framework do not rely on a narrowly tuned latent dimensionality, but persist across a broad range of compression levels.
Second, the results exhibit the expected compression--performance trade-off: increasing the latent dimension consistently improves detection accuracy, both in aggregate BER across the SNR sweep and in temporal BER throughout the adaptation trajectory. This behavior reflects the fact that larger latent spaces preserve a richer set of adaptation-relevant degrees of freedom, whereas excessive compression limits the representational capacity of the latent adaptation mechanism.
Comparing the two temporal dynamics parameterizations reveals that the learned $\myVec{F}$-based dynamics generally achieve stronger performance than the OU-based model when sufficient latent capacity is available, owing to their increased flexibility in modeling temporal evolution. However, both parameterizations display similar qualitative compression trends, suggesting that the benefits of latent Bayesian adaptation are robust to the particular state evolution model employed.

   \subsubsection{Cross-Dataset Generalization}
\label{subsec:cross_learning}

The wireless communication experiments in Section~\ref{sec:experiments} evaluate \ac{name} under matched train-test conditions, where both offline meta-training and online deployment are conducted using channel trajectories generated by the QuaDRiGa simulator. While this setting demonstrates the effectiveness of the proposed framework under realistic non-stationary fading, it does not assess the robustness of the learned latent adaptation space to distributional mismatch between offline and online environments.
To study this aspect, we consider a cross-dataset transfer setting in which \ac{name} is meta-trained offline on QuaDRiGa-generated channel trajectories and then deployed online on channels generated by COST2100~\cite{liu2012cost} (an alternative standardized stochastic wireless channel generator). This experiment evaluates whether the learned latent adaptation space captures transferable adaptation structure beyond the statistics of a single simulator.

\begin{figure}[t]
    \centering
    \begin{minipage}[t]{0.48\linewidth}
        \centering
        \includegraphics[width=\linewidth]{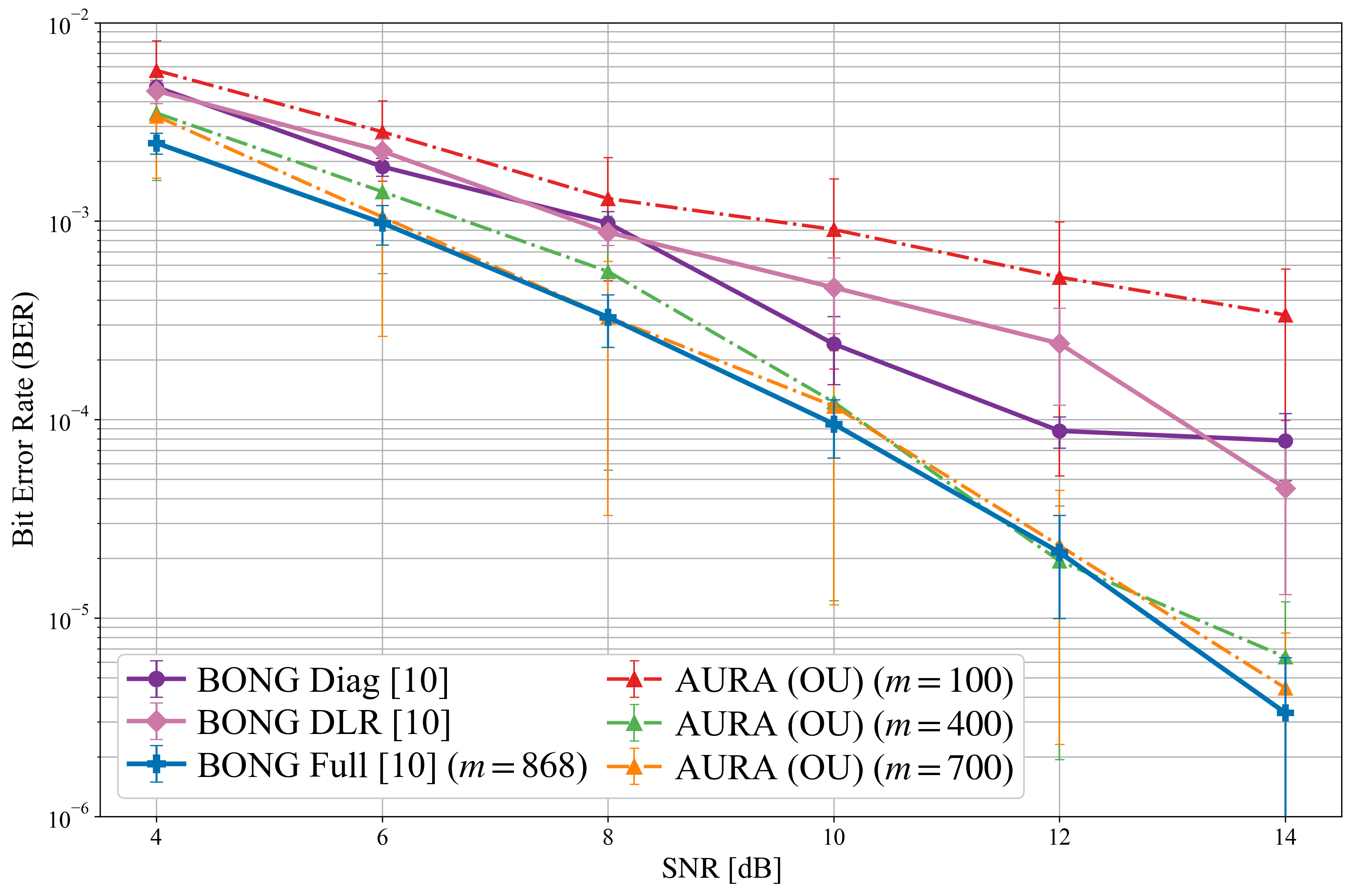}
        \caption{Cross-learning evaluation: BER vs. SNR.}
        \label{fig:cross_learning_snr}
    \end{minipage}
    \hfill
    \begin{minipage}[t]{0.48\linewidth}
        \centering
        \includegraphics[width=\linewidth]{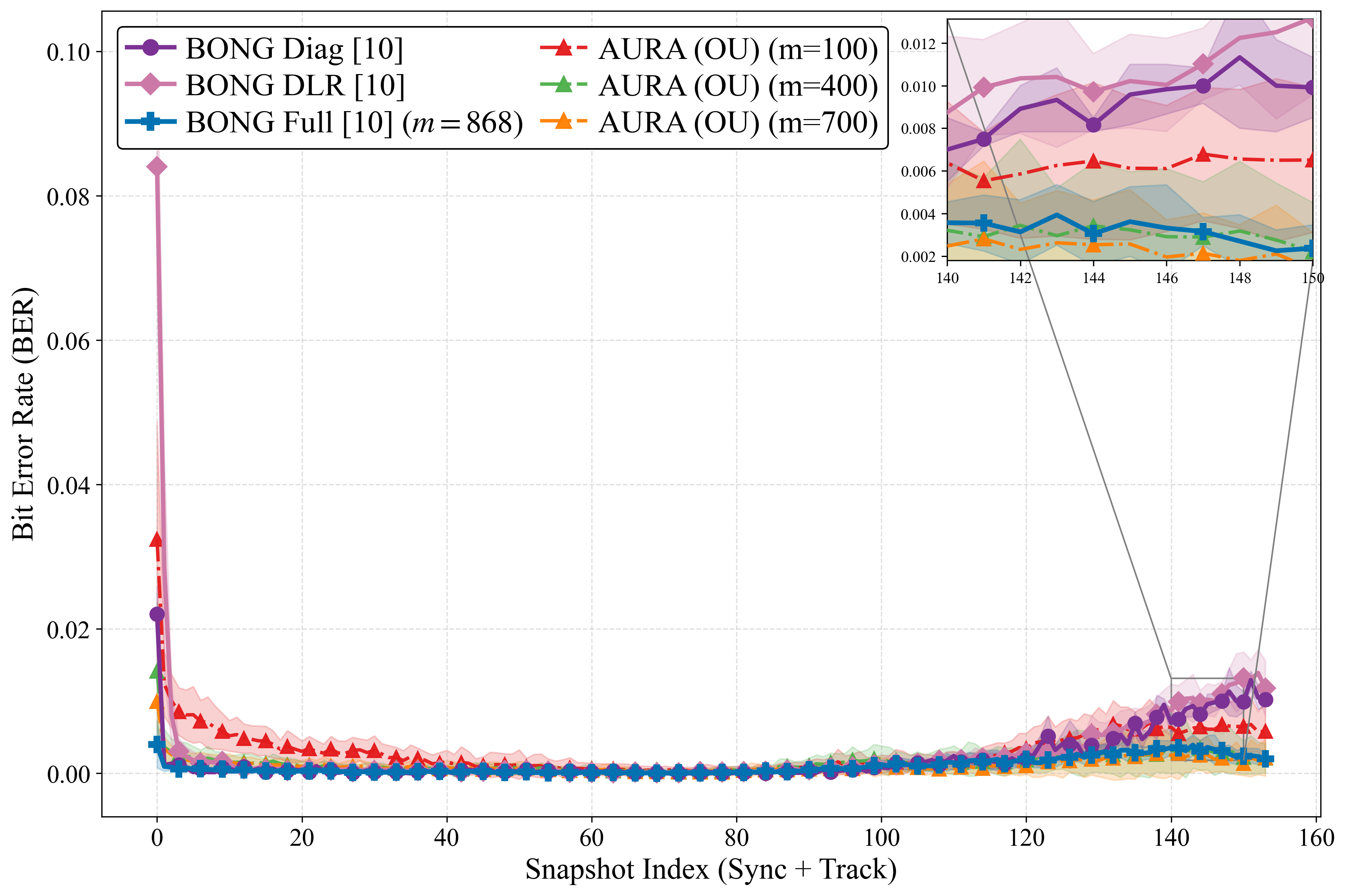}
        \caption{Cross-learning evaluation: BER vs. Snapshot Index.}
        \label{fig:cross_learning_snapshot}
    \end{minipage}
\end{figure}

The results reported in Figs.~\ref{fig:cross_learning_snr} and~\ref{fig:cross_learning_snapshot} show that \ac{name} remains effective under a pronounced train-test mismatch. In this cross-simulator setting, \ac{name} achieves BER performance that is broadly comparable to the leading full-parameter Bayesian baselines, while performing the online update in a substantially compressed latent space. Thus, although it does not provide a uniform BER advantage over the strongest full-space methods, the results indicate that the learned latent representation captures transferable adaptation structure beyond the simulator used for offline meta-training.



\subsubsection{Effect of Offline Meta-Learning.}
We further isolate the contribution of the offline meta-learning stage in the wireless receiver setting. To this end, we compare the full \ac{name} framework against a cold-start variant in which the same latent dimensionality and prediction model are used, but the latent representation is not learned offline. This comparison is carried out for both OU and diagonal-\(F\) prediction dynamics, while keeping the receiver architecture and online Bayesian update protocol fixed. Table~\ref{tab:offline_learning_comm} reports the tracking BER for two SNR values and different latent dimensions \(m\), where each warm-start/cold-start pair uses the same prediction model and latent dimension.

The results demonstrate that offline meta-learning is critical for reliable online tracking in the considered non-stationary communication scenario. Without offline learning, the latent-space update remains ineffective and yields high BER across prediction models and latent dimensions. In contrast, the warm-start \ac{name} variants consistently reduce the BER by orders of magnitude under the same online adaptation budget. This confirms that the observed gains do not stem merely from restricting the update to a lower-dimensional parameterization, but from learning an adaptation-aware latent geometry that aligns the online Bayesian updates with the dominant channel-induced variations.

\begin{table}[t]
\centering
\caption{Effect of offline learning in the QuaDRiGa communication scenario. Each pair compares the same prediction method and latent dimension with and without offline learning. BER in (\%)}
\label{tab:offline_learning_comm}
\small
\begin{tabular}{llccc}
\toprule
Prediction method & Latent dim. & Mode & SNR \(=8\) dB & SNR \(=10\) dB \\
\midrule
\multirow{2}{*}{OU} & \multirow{2}{*}{\(m=100\)} & warm-start & \(\mathbf{0.832 \pm 0.485}\) & \(\mathbf{0.512 \pm 0.379}\) \\
                    &                           & cold-start & \(14.399 \pm 3.024\) & \(11.542 \pm 1.054\) \\
\midrule
\multirow{2}{*}{OU} & \multirow{2}{*}{\(m=250\)} & warm-start & \(\mathbf{0.409 \pm 0.514}\) & \(\mathbf{0.292 \pm 0.418}\) \\
                    &                           & cold-start & \(14.662 \pm 1.541\) & \(12.225 \pm 1.259\) \\
\midrule
\multirow{2}{*}{OU} & \multirow{2}{*}{\(m=400\)} & warm-start & \(\mathbf{0.340 \pm 0.496}\) & \(\mathbf{0.249 \pm 0.353}\) \\
                    &                           & cold-start & \(14.828 \pm 2.880\) & \(12.320 \pm 0.643\) \\
\midrule
\multirow{2}{*}{OU} & \multirow{2}{*}{\(m=700\)} & warm-start & \(\mathbf{0.311 \pm 0.367}\) & \(\mathbf{0.210 \pm 0.152}\) \\
                    &                           & cold-start & \(14.870 \pm 1.626\) & \(12.980 \pm 1.487\) \\
\midrule
\multirow{2}{*}{F-Diag} & \multirow{2}{*}{\(m=100\)} & warm-start & \(\mathbf{0.0757 \pm 0.0079}\) & \(\mathbf{0.0608 \pm 0.0045}\) \\
                        &                           & cold-start & \(18.471 \pm 3.024\) & \(17.464 \pm 1.622\) \\
\midrule
\multirow{2}{*}{F-Diag} & \multirow{2}{*}{\(m=250\)} & warm-start & \(\mathbf{0.0763 \pm 0.0160}\) & \(\mathbf{0.0690 \pm 0.0148}\) \\
                        &                           & cold-start & \(18.677 \pm 1.541\) & \(17.508 \pm 1.810\) \\
\midrule
\multirow{2}{*}{F-Diag} & \multirow{2}{*}{\(m=400\)} & warm-start & \(\mathbf{0.0542 \pm 0.0014}\) & \(\mathbf{0.0510 \pm 0.0010}\) \\
                        &                           & cold-start & \(18.321 \pm 2.880\) & \(16.562 \pm 0.670\) \\
\midrule
\multirow{2}{*}{F-Diag} & \multirow{2}{*}{\(m=700\)} & warm-start & \(\mathbf{0.0199 \pm 0.0009}\) & \(\mathbf{0.0176 \pm 0.0004}\) \\
                        &                           & cold-start & \(18.225 \pm 1.626\) & \(17.027 \pm 1.815\) \\
\bottomrule
\end{tabular}
\end{table}

\subsubsection{Computational Cost versus Adaptation Accuracy.}
We complement the accuracy evaluation with a direct cost--accuracy analysis of the different adaptation rules.
Specifically, we compare the FLOPs required by each method against the resulting mean error at a fixed operating point of \(\mathrm{SNR}=12\) dB.
The error is measured on the linear QuaDRiGa channel, while the FLOPs are obtained using a dedicated cost-evaluation script tailored to the operations of each adaptation rule.
Diagonal covariance variants are evaluated using an element-wise implementation that exploits their structure, and BONG-DLR is evaluated with rank \(30\), as used in the reported experiments in \ref{fig:ber_compare_linear_mixed} and \ref{fig:ber_compare_nonlinear_mixed}.

Figure~\ref{fig:flops_vs_ber_snr12} reports the resulting trade-off.
The AURA variants occupy a favorable region of the cost-accuracy plane, achieving substantially lower error than the full-space BONG baseline while requiring fewer FLOPs.
Increasing the latent dimension provides a smooth trade-off between computational budget and adaptation quality.
The diagonal-\(F\) variant achieves the lowest error among the AURA configurations, with a moderate increase in cost relative to the OU dynamics.
In contrast, the diagonal and DLR BONG approximations reduce the cost of full-space Bayesian tracking but remain substantially less accurate in this setting.
These results indicate that adapting in a learned latent space preserves the relevant tracking directions while avoiding the computational burden of full-parameter Bayesian updates.

\begin{figure}[t]
    \centering
    \includegraphics[width=0.5\linewidth]{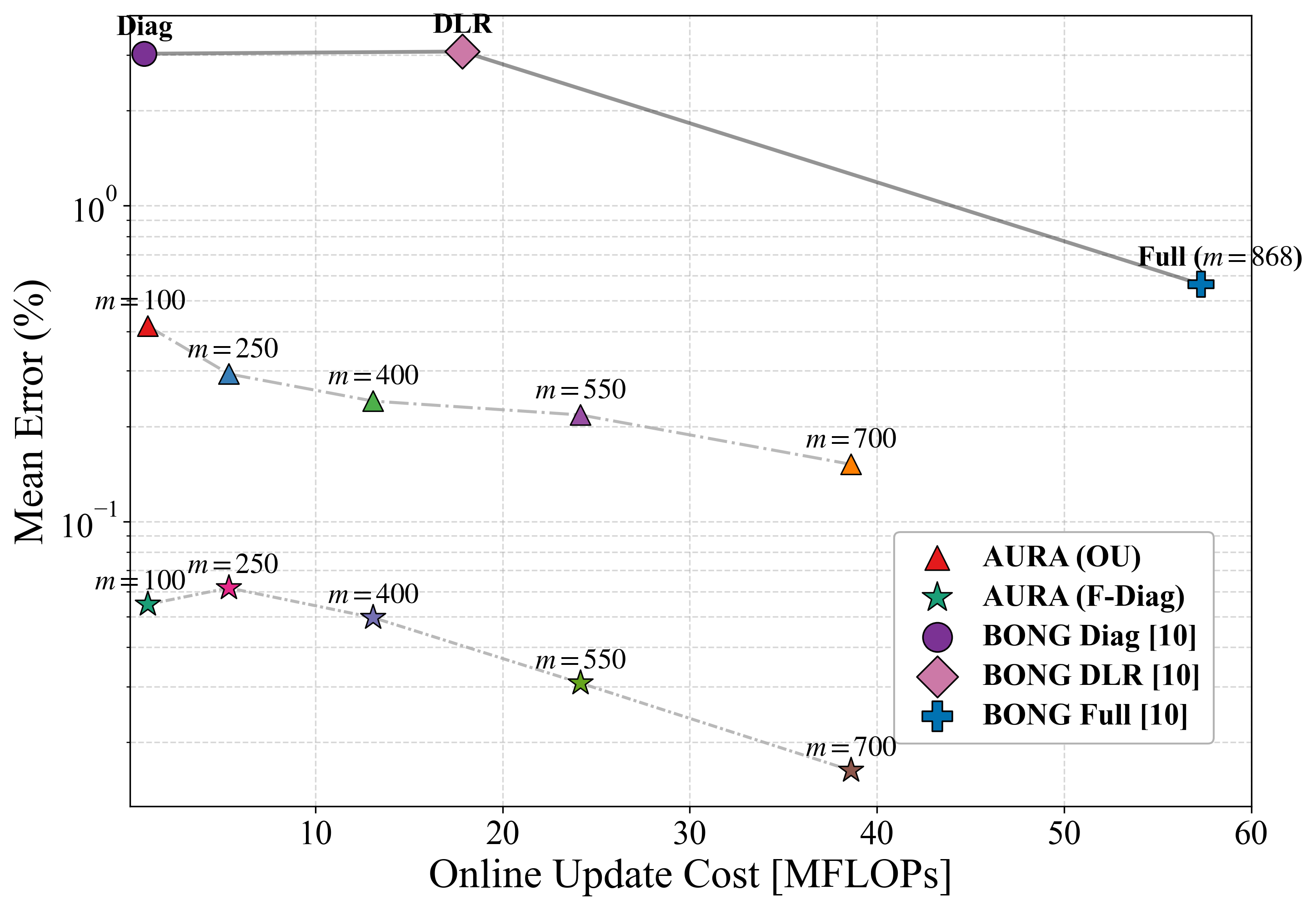}
    \caption{FLOPs-error trade-off on the linear QuaDRiGa channel at \(\mathrm{SNR}=12\) dB}
    \label{fig:flops_vs_ber_snr12}
\end{figure}

\subsubsection{Intermittent Supervision}
\label{subsec:intermittent_supervision}
The wireless experiments in Section~\ref{ssec:WirelessExp} assume that a small number of labeled pilot symbols is available at every frame. We next examine a more restrictive setting in which reliable supervision is available only intermittently. Specifically, we provide $K$ labeled pilots once every $I$ frames, with $I\in\{1,\ldots,5\}$ and $K\in\{2,4,6\}$. At supervised frames, the receiver performs the standard prediction--correction update using the available pilots. At intervening unlabeled frames, the correction step in \eqref{eq:update} is omitted, while the latent mean and covariance continue to evolve according to the prediction step in \eqref{eq:predict}. Importantly, the channel realization continues to vary at every frame; hence, increasing $I$ requires the receiver to track multiple successive channel variations without receiving an observation correction.

Table~\ref{tab:intermittent_supervision} reports the resulting \ac{ber} in $\%$ over all $15$ combinations of supervision interval and pilot budget.
\begin{table}[t]
\centering
\caption{BER under intermittent supervision (\%, mean $\pm$ standard deviation over three seeds).}
\label{tab:intermittent_supervision}
\small
\begin{tabular}{ccccc}
\toprule
$I$ & $K$ & \ac{name} & BONG Full & BONG Diag \\
\midrule
1 & 2 & $\mathbf{1.04 \pm 0.21}$ & $12.09 \pm 1.32$ & $17.02 \pm 1.48$ \\
1 & 4 & $\mathbf{0.10 \pm 0.02}$ & $3.15 \pm 0.35$ & $8.54 \pm 0.83$ \\
1 & 6 & $\mathbf{0.07 \pm 0.01}$ & $1.11 \pm 0.04$ & $4.14 \pm 0.23$ \\
\midrule
2 & 2 & $\mathbf{6.19 \pm 0.61}$ & $25.00 \pm 1.25$ & $29.38 \pm 0.65$ \\
2 & 4 & $\mathbf{0.70 \pm 0.07}$ & $12.17 \pm 0.38$ & $17.56 \pm 0.68$ \\
2 & 6 & $\mathbf{0.20 \pm 0.07}$ & $6.11 \pm 0.17$ & $12.18 \pm 0.28$ \\
\midrule
3 & 2 & $\mathbf{13.16 \pm 1.91}$ & $34.03 \pm 1.66$ & $37.60 \pm 1.22$ \\
3 & 4 & $\mathbf{3.07 \pm 0.37}$ & $19.55 \pm 1.17$ & $24.12 \pm 2.07$ \\
3 & 6 & $\mathbf{1.04 \pm 0.15}$ & $12.52 \pm 1.20$ & $18.46 \pm 1.88$ \\
\midrule
4 & 2 & $\mathbf{18.31 \pm 1.25}$ & $39.25 \pm 1.55$ & $39.29 \pm 0.38$ \\
4 & 4 & $\mathbf{5.28 \pm 0.27}$ & $25.21 \pm 0.40$ & $29.97 \pm 0.88$ \\
4 & 6 & $\mathbf{1.99 \pm 0.19}$ & $17.93 \pm 0.63$ & $22.11 \pm 0.94$ \\
\midrule
5 & 2 & $\mathbf{24.00 \pm 0.88}$ & $42.47 \pm 0.46$ & $43.21 \pm 1.26$ \\
5 & 4 & $\mathbf{10.60 \pm 1.08}$ & $29.15 \pm 0.39$ & $34.20 \pm 1.47$ \\
5 & 6 & $\mathbf{4.54 \pm 0.96}$ & $22.04 \pm 0.43$ & $27.48 \pm 2.05$ \\
\bottomrule
\end{tabular}
\end{table}
%
The results reveal two consistent trends. First, performance degrades as supervision becomes less frequent or fewer pilots are available, as expected when the receiver must rely more heavily on prediction between observation corrections. Second, \ac{name} achieves the lowest \ac{ber} in all $15$ evaluated configurations, with its advantage over parameter-space Bayesian tracking persisting as supervision becomes increasingly sparse.

These findings show that \ac{name} does not require an observation correction after every channel variation. The learned latent dynamics can propagate the adaptation state across several distribution changes, while subsequent labeled observations correct accumulated prediction errors. At the same time, the degradation observed for longer supervision intervals highlights the growing dependence on the accuracy of the learned dynamics. We therefore interpret these results as evidence that \ac{name} remains effective under moderate intermittent supervision, rather than as evidence of stability over arbitrarily long periods without feedback.

\subsubsection{Full Transition Matrix Ablation.}
We additionally evaluate an ablation in which the latent dynamics are parameterized by a fully learned transition matrix $\myMat{F}$, using the same linear QuaDRiGa setting considered in \eqref{eq:MIMO}. This variant provides a more expressive temporal model than the OU and diagonal-$\myMat{F}$ parameterizations, since it allows arbitrary linear coupling between latent coordinates. As shown in Fig.~\ref{fig:ablation_full_f}, the full-$\myMat{F}$ model can achieve strong tracking performance, particularly for sufficiently large latent dimensions.

However, this expressiveness comes with a substantially higher computational cost. In particular, the covariance prediction step in~\eqref{eq:predict} involves the multiplication
$\myMat{F}\myMat{\Sigma}\myMat{F}^{\top}$, whose complexity scales as $\mathcal{O}(m^3)$ for a dense transition matrix and covariance. Consequently, although the full-$\myMat{F}$ variant is informative as an ablation and can be competitive in terms of BER, it is only attractive from a runtime perspective for small latent dimensions, roughly $m\leq 250$ in our implementation. For larger latent states, the cubic prediction cost becomes inconsistent with the goal of fast online adaptation. This motivates our focus in the main experiments on the OU and diagonal-$\myMat{F}$ dynamics, which retain a structured temporal prior while reducing the covariance prediction cost relative to a full transition matrix.

\begin{figure}[t]
    \centering
    \begin{subfigure}[t]{0.48\linewidth}
        \centering
        \includegraphics[width=\linewidth]{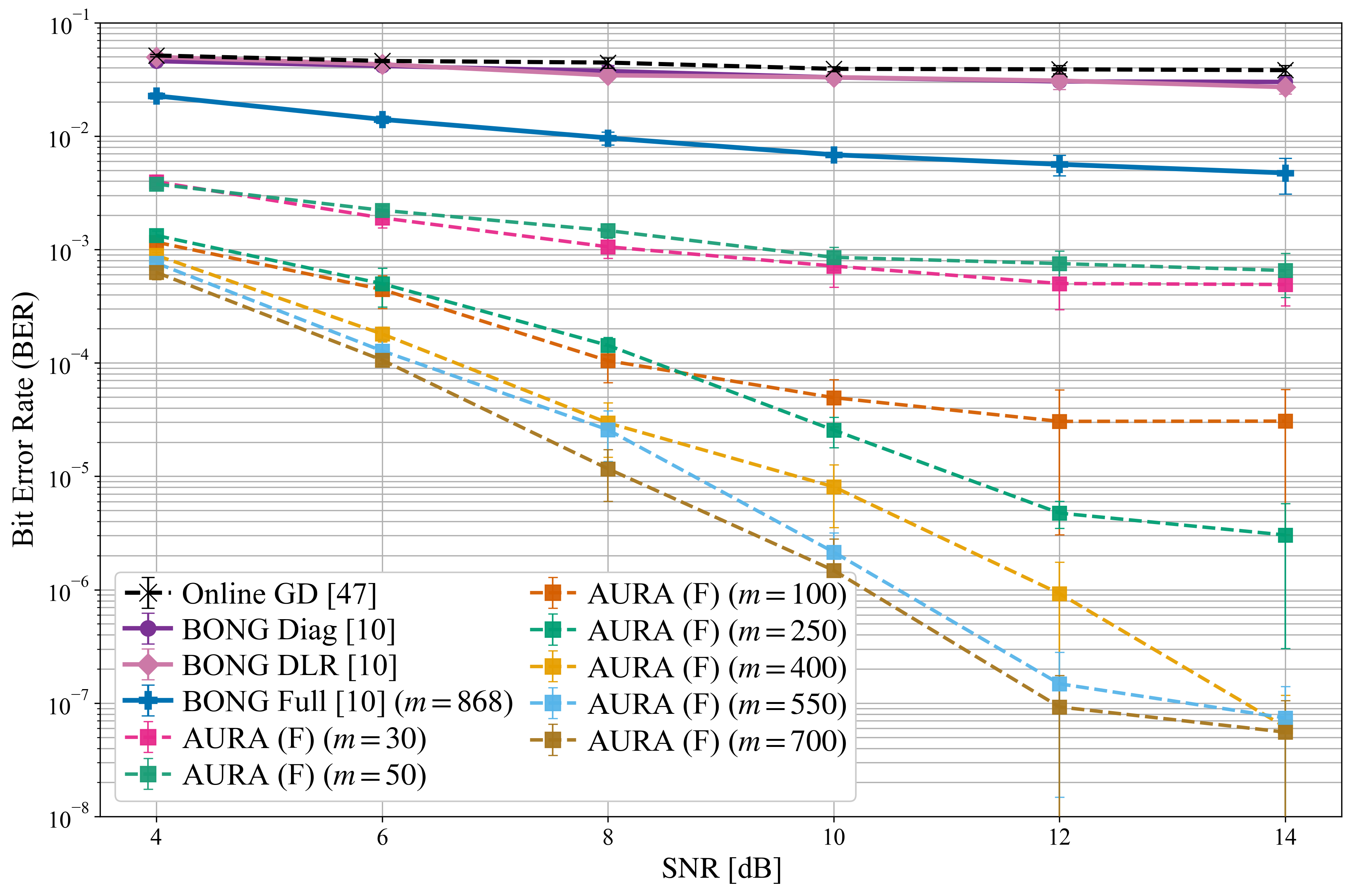}
        \caption{BER versus SNR.}
        \label{fig:ablation_full_f_snr}
    \end{subfigure}
    \hfill
    \begin{subfigure}[t]{0.48\linewidth}
        \centering
        \includegraphics[width=\linewidth]{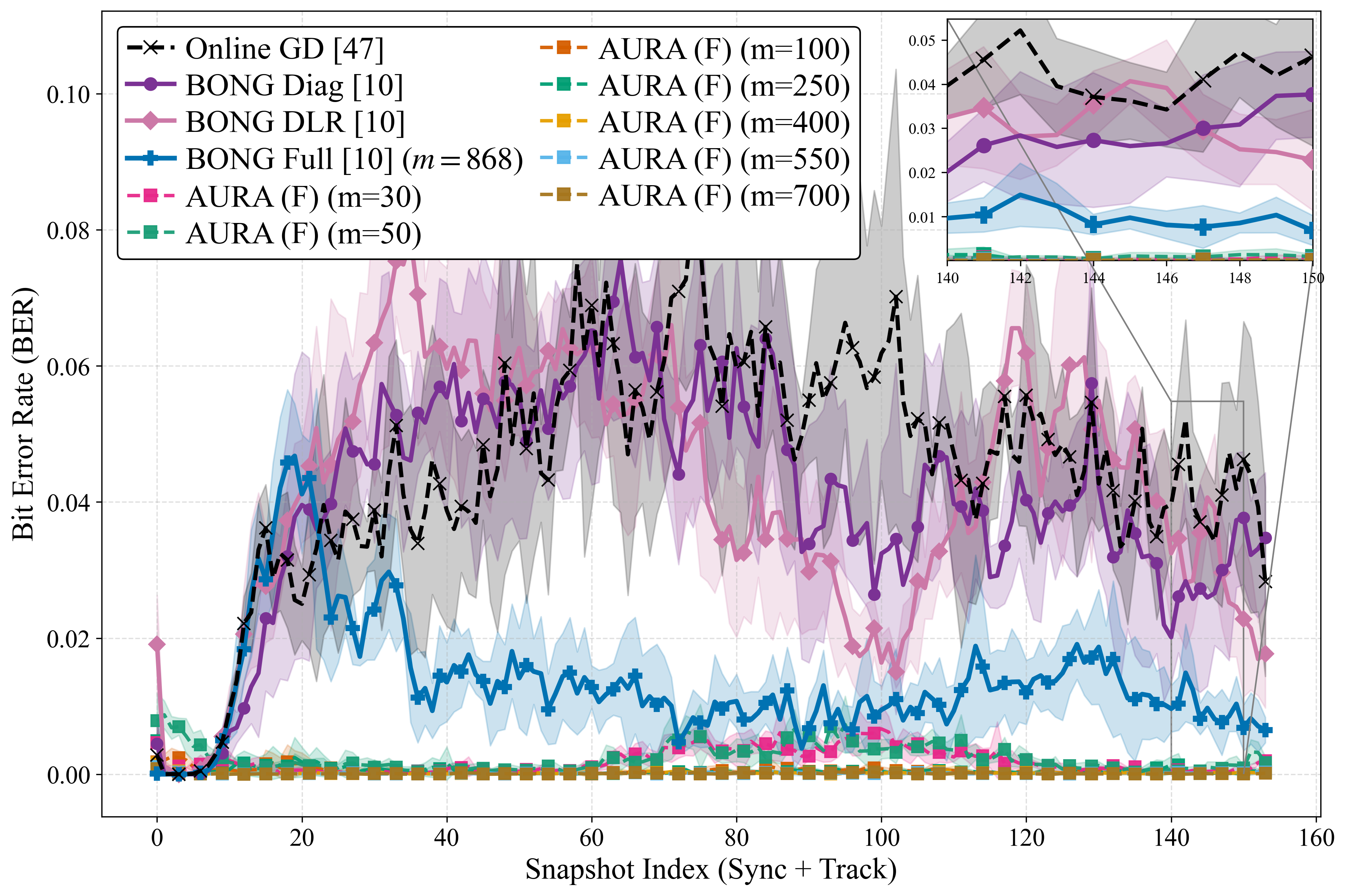}
        \caption{BER along the tracking horizon.}
        \label{fig:ablation_full_f_snapshot}
    \end{subfigure}
    \caption{Ablation of full-\(\myMat{F}\) latent dynamics on linear QuaDRiGa channel.}
    \label{fig:ablation_full_f}
\end{figure}

\subsection{Additional Non-stationary Image Classification Experiments}
\label{app:vision_ablation}

\subsubsection{Lifting Map Architecture}
\label{subsec:lifting_map_architecture}

A central component of \ac{name} is the learned lifting map
$\mathcal{G}:\mathbb{R}^{m}\rightarrow\mathbb{R}^{d}$, which determines how
the low-dimensional latent state is mapped to the adapted model parameters.
We next examine whether the effectiveness of the learned adaptation space
depends on restricting $\mathcal{G}$ to an affine transformation, or whether
additional nonlinear expressiveness is beneficial.

We compare two lifting-map parameterizations while keeping the latent
dimension and online adaptation protocol fixed. The affine variant corresponds
to
\[
\mathcal{G}_{\rm aff}(\myVec{z})
=
\myVec{\phi}+\myMat{A}\myVec{z},
\]
where $\myVec{\phi}$ denotes the source-model parameters and $\myMat{A}$ is
learned during offline meta-training. We additionally consider a residual
nonlinear lifting map of the form
\[
\mathcal{G}_{\rm nonlin}(\myVec{z})
=
\myVec{\phi}
+
\myMat{A}\myVec{z}
+
\myVec{\psi}(\myVec{z}),
\]
where $\myVec{\psi}(\cdot)$ is a learned MLP with a single hidden layer of
width $500$ and GELU activation.  Both variants are subsequently meta-learned end-to-end through
\ref{alg:meta_learning}

Table~\ref{tab:lifting_map_architecture} reports the classification error on
CIFAR-10-C and CIFAR-100-C for different numbers of labeled adaptation
samples.
\begin{table}[t]
\centering
\caption{Effect of the lifting-map architecture on classification error
(\%, mean $\pm$ standard deviation over three seeds).}
\label{tab:lifting_map_architecture}
\small
\begin{tabular}{lccc}
\toprule
Dataset & \# Labeled Samples & Affine & Nonlinear \\
\midrule
\multirow{3}{*}{CIFAR-10-C}
& 5  & $16.13 \pm 1.50$ & $\mathbf{10.21 \pm 0.40}$ \\
& 10 & $16.84 \pm 0.51$ & $\mathbf{10.15 \pm 0.81}$ \\
& 15 & $14.28 \pm 2.87$ & $\mathbf{9.75 \pm 0.54}$ \\
\midrule
\multirow{3}{*}{CIFAR-100-C}
& 5  & $29.47 \pm 0.62$ & $\mathbf{26.34 \pm 1.77}$ \\
& 10 & $29.65 \pm 1.97$ & $\mathbf{25.21 \pm 1.04}$ \\
& 15 & $26.66 \pm 3.89$ & $\mathbf{23.00 \pm 0.79}$ \\
\bottomrule
\end{tabular}
\end{table}
The results highlight the importance of both learning and appropriately parameterizing the lifting map. As shown by the communication ablation in Table~\ref{tab:offline_learning_comm}, merely restricting adaptation to a low-dimensional parameterization is not sufficient: without offline learning of the latent representation, Bayesian tracking fails to provide effective adaptation. Once the lifting geometry is learned, however, even an affine map yields strong performance. Increasing the expressiveness of $\mathcal{G}(\cdot)$ through a nonlinear parameterization further improves accuracy consistently across both CIFAR-10-C and CIFAR-100-C and across all considered supervision budgets. These results suggest that the learned latent geometry captures the dominant adaptation directions, while nonlinear lifting provides additional flexibility for representing the relation between the latent state and the evolving model parameters.

This gain in expressiveness introduces a corresponding computational trade-off. Each \ac{ekf} correction requires evaluating the observation Jacobian and therefore differentiating through $\mathcal{G}(\cdot)$, while parameter reconstruction additionally requires its forward evaluation. Consequently, deeper or more expressive lifting maps increase both forward and backward computational costs during online adaptation. The lifting-map architecture therefore provides a direct accuracy-latency design choice: affine mappings offer a simpler and more efficient realization, whereas nonlinear mappings can provide improved adaptation accuracy at higher computational cost. This overhead could potentially be reduced through dedicated hardware or compilation and optimized execution of the lifting and Jacobian computations; such implementation-level optimizations are not investigated in this work.

\subsubsection{Cross-Dataset Generalization}
The non-stationary image classification experiments in Section~\ref{sec:experiments} evaluate \ac{name} under gradual corruption streams generated from the same corruption pool used during offline meta-training. Although the offline and online phases remain distinct due to different temporal mixing trajectories, this setting does not test whether the learned latent adaptation space generalizes beyond the corruption families observed during training.
To assess this aspect, we introduce an explicit corruption-family split between offline and online phases. Specifically, \ac{name} is meta-trained offline using one subset of corruption types and subsequently deployed online on a disjoint set of previously unseen corruptions. This setting evaluates whether the learned latent adaptation mechanism captures transferable adaptation structure beyond the specific degradation families encountered during offline training.

\begin{table*}[t]
\centering
\caption{Cross-corruption generalization under mixed-gradual CIFAR-10.}
\label{tab:cross_corruption}
\renewcommand{\arraystretch}{1.15}
\setlength{\tabcolsep}{6pt}
\resizebox{\linewidth}{!}{%
\begin{tabular}{lcccccccc}
\hline
\textbf{Method} 
& \textbf{ROID} 
& \textbf{AURA} 
& \textbf{BN-EMA} 
& \textbf{BN-Test} 
& \textbf{EKF-FC} 
& \textbf{BN-Alpha} 
& \textbf{Source} 
& \textbf{Online GD} \\
\hline
\textbf{Classification Error (\%)}
& $9.88 \pm 1.01$
& $9.99 \pm 1.30$
& $11.50 \pm 2.64$
& $11.62 \pm 1.29$
& $12.31 \pm 1.02$
& $13.55 \pm 2.47$
& $17.07 \pm 2.67$
& $16.21 \pm 2.04$ \\
\hline
\end{tabular}
}
\end{table*}


The results in Table~\ref{tab:cross_corruption} show that \ac{name} remains highly effective under this transfer setting, achieving strong performance while remaining competitive with the best-performing baseline and clearly outperforming the remaining alternatives. This indicates that the learned latent adaptation space captures adaptation-relevant structure that transfers across distinct corruption families and temporal drift patterns.
These findings complement the main-paper image classification results by demonstrating that the adaptation mechanisms learned by \ac{name} generalize beyond the precise distribution shifts observed during offline meta-training, further supporting the robustness and transferability of the proposed meta-learned latent adaptation framework. 


\subsubsection{Comparison with Efficient Adaptation Methods}
\label{subsec:efficient_adaptation}

A central motivation of \ac{name} is to make rapid online adaptation
computationally tractable by operating in a compact learned representation.
We therefore compare against two complementary classes of efficient
adaptation methods. First, we consider LoRA~\cite{hu2022lora} as a
parameter-efficient supervised adaptation baseline, asking whether restricting
updates to a compact low-rank parameterization can reproduce the gains of the
learned latent representation. Second, we compare against the efficiency-oriented
TTA methods EcoTTA~\cite{song2023ecotta} and
ELaTTA~\cite{luo2025elatta}, which operate without ground-truth supervision
and explicitly target low-cost deployment.

\paragraph{Comparison with LoRA.}
We evaluate LoRA~\cite{hu2022lora} under the same supervised streaming
protocol as \ac{name}. Following a hyperparameter sweep, we use rank $r=8$
and adapt the classification head together with selected convolutional
layers: the first convolutional stage for CIFAR-10-C and the second stage
for CIFAR-100-C. We compare LoRA with both affine and nonlinear realizations
of the learned lifting map.

\begin{table}[t]
\centering
\caption{Classification error (\%, mean $\pm$ standard deviation over three
seeds) for LoRA and \ac{name}.}
\label{tab:lora_comparison}
\small
\begin{tabular}{lccc}
\toprule
Dataset & LoRA & \ac{name}, Affine & \ac{name}, Nonlinear \\
\midrule
CIFAR-10-C
& $15.73 \pm 1.23$
& $14.28 \pm 2.87$
& $\mathbf{9.75 \pm 0.54}$ \\
CIFAR-100-C
& $28.90 \pm 1.63$
& $26.66 \pm 3.89$
& $\mathbf{23.00 \pm 0.79}$ \\
\bottomrule
\end{tabular}
\end{table}

The comparison with LoRA isolates the effect of the adaptation
parameterization under supervised streaming conditions. Both \ac{name}
variants improve upon LoRA, while the nonlinear lifting map provides the
largest gain. This indicates that restricting adaptation to a compact
parameterization alone does not account for the performance of \ac{name};
learning the adaptation representation from non-stationary trajectories
provides an additional benefit.

\paragraph{Comparison with Efficient TTA.}
We further compare against EcoTTA~\cite{song2023ecotta} and
ELaTTA~\cite{luo2025elatta} on the gradual mixed-corruption streams used in
the main image-classification experiments. These methods retain their
respective unsupervised adaptation procedures, whereas \ac{name} uses five
labeled samples per update. The comparison therefore reflects distinct
supervision regimes and is intended to characterize the resulting
accuracy--efficiency trade-off rather than provide a supervision-matched
benchmark.

\begin{table}[t]
\centering
\caption{Classification error (\%) and adaptation latency for efficient
test-time adaptation methods.}
\label{tab:efficient_tta}
\small
\begin{tabular}{lcccc}
\toprule
Method & Supervision & CIFAR-10-C & CIFAR-100-C & Latency (ms/sample) for CIFAR-100-C) \\
\midrule
Source
& None
& 41.81
& 37.41
& -- \\
\ac{name}
& 5 labeled
& \textbf{16.13}
& \textbf{29.47}
& 2.01 \\
EcoTTA~\cite{song2023ecotta}
& Unlabeled
& 20.18
& 41.20
& 1.15 \\
ELaTTA~\cite{luo2025elatta}
& Unlabeled
& 44.82
& 39.98
& 0.41 \\
\bottomrule
\end{tabular}
\end{table}

EcoTTA and ELaTTA avoid labeled feedback and achieve lower adaptation
latency, but incur higher classification error on the considered streams.
EcoTTA remains competitive on CIFAR-10-C but does not improve over the
source model on CIFAR-100-C, while ELaTTA does not improve upon the source
model in either setting. These results place \ac{name} at a different
operating point in the supervision--accuracy--efficiency trade-off: it uses
sparse reliable supervision to obtain substantially stronger adaptation
accuracy while retaining low online computational cost.

\subsubsection{Robustness to Abrupt Distribution Shifts}
\label{subsec:abrupt_shifts}
The main image-classification experiments consider gradually evolving
distribution shifts, which are naturally aligned with the temporal model
learned by \ac{name}. We therefore evaluate whether this learned dynamical
prior remains useful when the deployment distribution changes abruptly,
despite offline meta-training being performed exclusively on gradual
mixed-corruption trajectories.

We evaluate CIFAR-10-C at severity~$5$ under five abrupt corruption shifts:
pixelation, contrast, zoom blur, fog, and JPEG compression. Following each
shift, the new corruption is maintained for $40$ online update steps, with
$15$ labeled adaptation samples available at each step. Performance is
measured immediately after the transition and after $10$, $20$, and $40$
updates. We evaluate both affine and nonlinear lifting maps using the same
latent dimension, $m=500$.

\begin{table}[t]
\centering
\caption{Recovery from abrupt distribution shifts on CIFAR-10-C. Mean
classification error (\%) across five corruption shifts.}
\label{tab:abrupt_shifts}
\small
\begin{tabular}{lcccc}
\toprule
Lifting map
& After shift & 10 updates & 20 updates & 40 updates \\
\midrule
Affine
& 45.96 & 27.68 & 21.74 & 19.44 \\
Nonlinear
& 42.44 & 23.30 & 19.05 & \textbf{17.66} \\
\bottomrule
\end{tabular}
\end{table}

As shown in Table~\ref{tab:abrupt_shifts}, both variants exhibit substantial
recovery following the abrupt transition, despite the mismatch between the
deployment dynamics and those used for offline meta-training. The nonlinear
lifting map consistently attains lower error throughout the recovery horizon,
while both parameterizations reduce the initial post-shift error by
approximately $58\%$ after $40$ updates.

These results indicate that the learned temporal dynamics act as a predictive
prior rather than imposing a rigid smoothness constraint on the adaptation
trajectory. When an abrupt shift invalidates the prediction, subsequent
innovation-based \ac{ekf} corrections can progressively move the latent state
toward the new regime. Thus, while gradual variation remains the setting most
naturally matched to \ac{name}, the learned latent dynamics do not prevent
effective recovery from abrupt distribution changes.

\section{Proof-of-Concept: Adaptive Deep Software-Defined Radio}
\label{app:demo} 

To complement the simulation-based results in Section~\ref{sec:experiments}, we provide a proof-of-concept hardware demonstration validating the practical deployability of \ac{name} in a real-time communication system. Specifically, we implement a self-adaptive neural WiFi receiver using \ac{name}-style online adaptation within an over-the-air IEEE 802.11a communication link realized on software-defined radio hardware. This demonstration is intended to assess whether the computational efficiency of the proposed latent Bayesian adaptation framework translates into practical millisecond-order online adaptation under realistic hardware and latency constraints. As shown below, the resulting prototype confirms that \ac{name} enables real-time adaptive neural reception on commodity RF platforms using CPU-only inference and adaptation.

\subsection{Hardware Setup and Adaptive Receiver Architecture}
\label{app:demo_setup}
The proof-of-concept system implements an over-the-air IEEE 802.11a OFDM WiFi link using PlutoSDR+ devices as RF front-ends. The deployed communication setup follows a standard packetized WiFi transmission structure with 64 OFDM subcarriers, of which 52 are active (48 data and 4 pilot subcarriers). We consider a block-fading regime in which the wireless channel is approximately constant within each packet but varies continuously across packets, inducing the non-stationary adaptation setting targeted by \ac{name}.

\begin{figure}
    \centering
    \includegraphics[width=\linewidth]{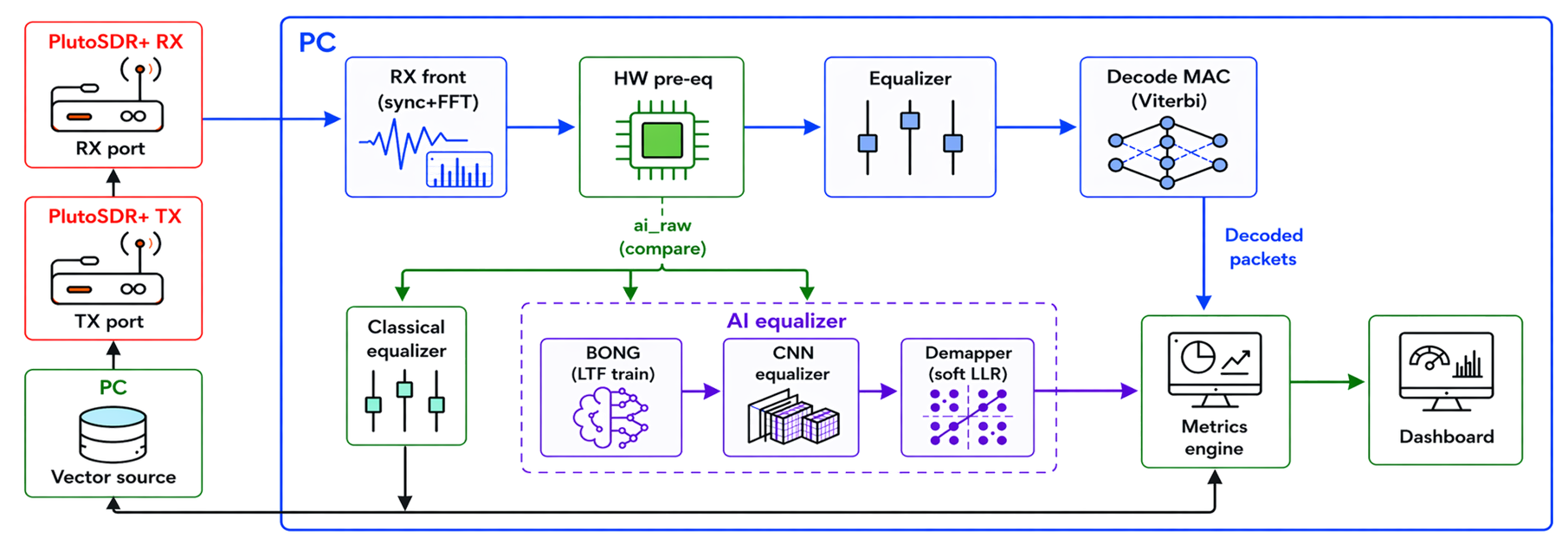}
    \caption{Hardware proof-of-concept setup for the self-adaptive neural WiFi receiver. The adaptive receiver is deployed inline with a conventional classical receiver within an over-the-air IEEE 802.11a communication link implemented using PlutoSDR+ RF front-ends, enabling simultaneous evaluation under identical channel realizations.}
    \label{fig:hardware_setup}
\end{figure}

A block diagram of the experimental platform is shown in Fig.~\ref{fig:hardware_setup}. The adaptive receiver consists of a lightweight learnable 1D-CNN equalizer containing approximately 1.8K trainable parameters, followed by an analytical modulation-aware demapper that produces soft bit log-likelihood ratios for downstream decoding. Online adaptation is performed using the known WiFi preamble (long-training-field symbols) available at the beginning of each packet, enabling the receiver to refine its parameters prior to processing the data-bearing payload symbols.

To enable controlled comparison, the adaptive neural receiver is executed inline alongside a conventional  classical receiver based on LMMSE equalization, with both receivers processing the same received packets simultaneously. This allows direct measurement of the corresponding bit-error rates under identical channel realizations. Importantly, the online adaptation pipeline is implemented entirely in NumPy and executed on CPU without GPU acceleration.

\subsection{Hardware Performance and Runtime Characterization}
\label{app:demo_results}

The resulting hardware performance is illustrated in Fig.~\ref{fig:hardware_results}, which compares the measured bit-error rate of the adaptive neural receiver against that of the conventional classical receiver across varying \ac{snr} and \ac{cfo} conditions. The results show that augmenting the receiver with online neural adaptation substantially enlarges the reliable operating regime, particularly under challenging low-SNR and high-CFO conditions where the classical receiver begins to degrade significantly. In several operating regions, the adaptive neural receiver achieves markedly improved detection performance, demonstrating that the learned adaptive equalization mechanism provides tangible robustness benefits in real over-the-air deployments.

\begin{figure}
    \centering
    \includegraphics[width=\linewidth]{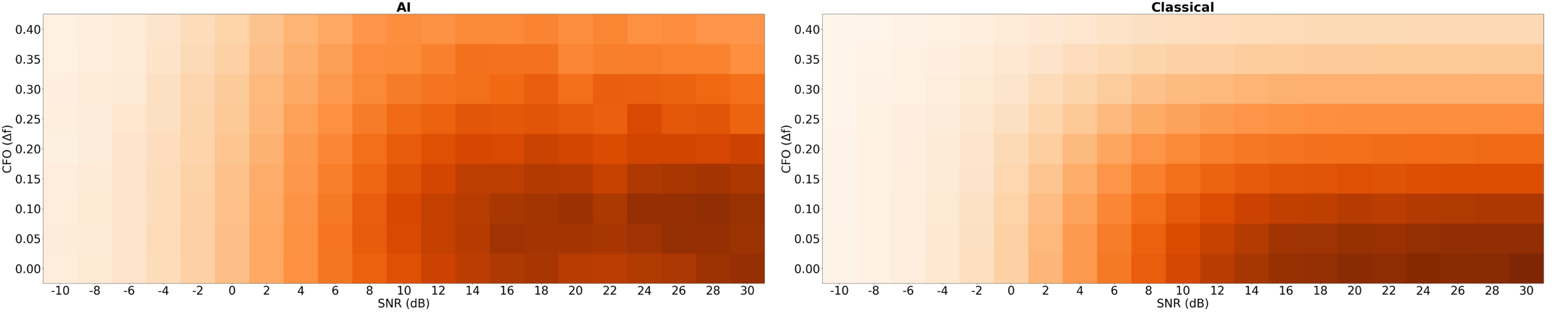}
    \caption{Average \ac{ber} heatmaps versus \ac{snr} and \ac{cfo} level for neural receiver adapted with \ac{name} (left) compared to classical channel estimation + LMMSE equalization receiver (right)}
    \label{fig:hardware_results}
\end{figure}

To assess deployment feasibility, Table~\ref{tab:hardware_runtime} reports the measured per-packet runtime of the adaptive receiver pipeline for multiple neural equalizer architectures of increasing complexity. Specifically, we report both inference latency and online adaptation latency for the lightweight 1.8K-parameter CNN used in the proof-of-concept, as well as for two deeper CNN variants. All runtimes are measured on CPU without GPU acceleration.

\begin{table}[t]
\centering
\caption{Measured per-packet runtime of the adaptive neural receiver for different CNN equalizer architectures. All runtimes are reported in milliseconds and measured using CPU-only execution.}
\label{tab:hardware_runtime}
\small
\begin{tabular}{lcc}
\toprule
Architecture & Inference Runtime [ms] & Adaptation Runtime [ms] \\
\midrule 
1.8K-Parameter CNN        & 0.084 & 10.573 \\
6.6k-Parameter CNN                  & 0.113 & 40.414 \\
8.2k-Parameter CNN                & 0.221 & 84.393 \\
\bottomrule
\end{tabular}
\end{table}

The measured runtimes confirm that \ac{name}-style online adaptation remains feasible within millisecond-order latency budgets even when deployed on modest CPU hardware. In particular, the lightweight neural receiver used in the proof-of-concept performs full online adaptation in tens of milliseconds without hardware acceleration, while larger models remain deployable at increased but still practical latency. These findings validate that the computational structure underlying \ac{name} enables real-world deployment under realistic hardware constraints, and suggest that further latency reductions are attainable through dedicated optimized implementations or edge-accelerated inference.

\section{Limitations}
\label{app:limitations}
While \ac{name} provides an efficient framework for Bayesian online adaptation via latent-space filtering, several limitations should be acknowledged:

\begin{itemize}
    \item 
    Although \ac{name} substantially reduces the dimensionality of the online adaptation problem, scaling to extremely large models or broad parameter blocks remains challenging. Compressing very high-dimensional parameter spaces into a single latent state introduces memory and optimization bottlenecks, and requires balancing latent compactness against sufficient expressiveness for accurate parameter reconstruction and adaptation.

    \item
    The effectiveness of \ac{name} depends not only on learning a suitable latent representation, but also on accurately modeling the temporal evolution of the latent state. Since the prediction step propagates the latent estimate according to an assumed or learned transition model, performance may degrade when deployment-time dynamics differ substantially from those observed during offline training. While the correction step can compensate using new observations, inaccurate predictions may reduce adaptation quality and stability.

    \item 
    The proposed state-space formulation is naturally best suited for gradually varying environments, where the previous latent state provides an informative prior for the next update. Under abrupt distribution shifts or regime changes, the prediction model may become temporarily misleading, potentially requiring complementary mechanisms such as change-point detection, covariance inflation, or latent-state reset strategies.

    \item 
    In its current form, \ac{name} assumes access to labeled samples during online adaptation. While this is natural in several streaming settings (e.g., pilot-aided communications), it may limit applicability in fully unsupervised deployment scenarios. Extending the framework to unlabeled adaptation via pseudo-labeling or self-supervised objectives is possible in principle, but may introduce sensitivity to label noise and error accumulation.

    \item \textbf{Approximate Filtering Assumptions.}
    The online adaptation rule relies on first-order linearization of the nonlinear observation model and Gaussian approximations inherent to the \ac{ekf}. As with other Kalman-style methods, significant deviations from these assumptions may reduce the fidelity of the Bayesian update approximation.
\end{itemize}

These limitations suggest that \ac{name} is most suitable for applications involving structured and moderately smooth non-stationarity with available online supervision, while improving robustness to abrupt shifts, unsupervised deployment, and larger-scale models remains an important direction for future work.



\end{document}